\documentclass[journal]{IEEEtran}
\usepackage[utf8]{inputenc}
\usepackage{cite}
\usepackage{graphicx}
\usepackage{subfigure}
\usepackage{amsmath,amssymb,amsfonts}
\usepackage{color}
\usepackage{bm}
\usepackage{algorithm}
\usepackage{algorithmic}
\usepackage{etoolbox}
\usepackage{enumerate}
\usepackage{amsthm}
\usepackage{multirow}
\usepackage{url}
\newtheorem{definition}{Definition}

\newtheorem{lemma}{Lemma}
\newtheorem{theorem}{Theorem}

\allowdisplaybreaks[4]

\makeatletter
\patchcmd{\@makecaption}
  {\scshape}
  {}
  {}
  {}
\makeatletter
\patchcmd{\@makecaption}
  {\\}
  {.\ }
  {}
  {}
\makeatother

\makeatletter
\long\def\@makecaption#1#2{%
  \vskip\abovecaptionskip
  \sbox\@tempboxa{\normalfont\footnotesize #1: #2}%
  \ifdim \wd\@tempboxa >\hsize
    {\normalfont\footnotesize #1: #2\par}
  \else
    \global \@minipagefalse
    \hb@xt@\hsize{\box\@tempboxa\hfil}%
  \fi
  \vskip\belowcaptionskip}
\makeatother

\renewcommand{\figurename}{Figure}
\def\tablename{Table}

\begin{document}

\title{Kronecker CP Decomposition with Fast Multiplication for Compressing RNNs}

\author{Dingheng Wang*, Bijiao Wu*, Guangshe Zhao, Man Yao, \\ Hengnu Chen, Lei Deng,~\IEEEmembership{Member,~IEEE}, Tianyi Yan\dag, Guoqi Li\dag,~\IEEEmembership{Member,~IEEE}
\thanks{*: These authors contribute equally to this work.}
\thanks{D. Wang, B. Wu, G. Zhao and M. Yao are with School of Automation Science and Engineering, Faculty of Electronic and Information Engineering, Xi'an Jiaotong University, Xi'an, Shaanxi 710049, China.}
\thanks{H. Chen and G. Li are with Department of Precision Instrumentation, Center for Brain Inspired Computing Research and Beijing Innovation Center for Future Chip, Tsinghua University, Beijing 100084, China.}
\thanks{L. Deng is with University of California, Santa Barbara, CA93106, USA.}
\thanks{T. Yan is with School of Life Science, Beijing Institute of Technology, Beijing 100084, China.}
\thanks{\dag: Corresponding authors, \protect\url{zhaogs@mail.xjtu.edu.cn}, \protect\url{liguoqi@mail.tsinghua.edu.cn}.}}
\maketitle

\begin{abstract}
Recurrent neural networks (RNNs) are powerful in the tasks oriented to sequential data, such as natural language processing and video recognition. However, since the modern RNNs, including long-short term memory (LSTM) and gated recurrent unit (GRU) networks, have complex topologies and expensive space/computation complexity, compressing them becomes a hot and promising topic in recent years. Among plenty of compression methods, tensor decomposition, e.g., tensor train (TT), block term (BT), tensor ring (TR) and hierarchical Tucker (HT), appears to be the most amazing approach since a very high compression ratio might be obtained. Nevertheless, none of these tensor decomposition formats can provide both the space and computation efficiency. In this paper, we consider to compress RNNs based on a novel Kronecker CANDECOMP/PARAFAC (KCP) decomposition, which is derived from Kronecker tensor (KT) decomposition, by proposing two fast algorithms of multiplication between the input and the tensor-decomposed weight. According to our experiments based on UCF11, Youtube Celebrities Face and UCF50 datasets, it can be verified that the proposed KCP-RNNs have comparable performance of accuracy with those in other tensor-decomposed formats, and even 278,219\(\times\) compression ratio could be obtained by the low rank KCP. More importantly, KCP-RNNs are efficient in both space and computation complexity compared with other tensor-decomposed ones under similar ranks. Besides, we find KCP has the best potential for parallel computing to accelerate the calculations in neural networks.
\end{abstract}

\begin{IEEEkeywords}
Kronecker Tensor Decomposition, Kronecker CP Decomposition, Fast Multiplication, Network Compression, Recurrent Neural Networks
\end{IEEEkeywords}

%%%%%%%%%%%%%%%%%%%%%%%%%%%%%%%%%%%%%%%%%%%%%%%%%%%%%%%%%%%%%%%%%%%%%%%%%%%%%%%%%%%%%%%%%%%
\section{Introduction}\label{sec:Intro}

It is commonly believed that the recurrent neural network (RNN) is one of the most powerful instruments to deal with the data processing tasks that refer to sequential structures, such as language processing and video recognition. Among multiple variants of RNNs, long-short term memory (LSTM) \cite{Hochreiter_1997_LSTMInvent} and gated recurrent unit (GRU) \cite{Cho_2014_GRUInvent} are the two most widely used architectures nowadays, since they can avoid the intractable vanishing gradient problem \cite{Bengio_1994_VanishGradient}. However, both of these kinds of RNNs have complex internal topologies. Making LSTM as an example, each of its units has 4 gated connections like \cite{Greff_2017_LSTM}
\begin{equation}\label{Eq_lstm}
\begin{aligned}
&\bm{y}_{f}=\sigma(\bm{W}_{f}\bm{x}(t)+\bm{U}_{f}\bm{h}(t-1)+\bm{b}_{f}) \\
&\bm{y}_{i}=\sigma(\bm{W}_{i}\bm{x}(t)+\bm{U}_{i}\bm{h}(t-1)+\bm{b}_{i}) \\
&\bm{y}_{z}=\sigma(\bm{W}_{z}\bm{x}(t)+\bm{U}_{z}\bm{h}(t-1)+\bm{b}_{z}) \\
&\bm{y}_{o}=\sigma(\bm{W}_{o}\bm{x}(t)+\bm{U}_{o}\bm{h}(t-1)+\bm{b}_{o}) \\
&\bm{c}(t) = \bm{y}_{f} \odot \bm{c}(t-1) + \bm{y}_{i} \odot \bm{y}_{z} \\
&\bm{h}(t)=\bm{y}_{o} \odot \sigma(\bm{c}(t))
\end{aligned}
\end{equation}
where \(\bm{x}(t)\) is the input vector , \(\bm{h}(t-1)\) denotes the status of previous time, \(\bm{c}(t)\) represents the long time memory, \(\odot\) is the element-wise product, \(\bm{W}_{\theta}\), \(\bm{U}_{\theta}\), \(\bm{b}_{\theta}\), and \(\bm{y}_{\theta}\) (\(\theta \in \{f,i,z,o\}\)) are respectively the input matrices, state matrices, biases and output vectors. Evidently, the volume of matrices in modern RNNs is the fountainhead of their considerable high space complexity. Therefore, many researchers have aimed their attention to compressing RNNs \cite{Deng_2020_Survey}.

Commonly, there are four aspects can be concluded to compress RNNs \cite{Deng_2020_Survey}: 1) compact design which optimizes the redundant topologies in the architecture of RNNs, e.g., S-LSTM \cite{Wu_2016_SLSTM} and skip-connected RNN \cite{Zhang_2016_SkipRNN}; 2) quantizing data from high-precision value space to low-precision value space such as integer field \(\mathbb{N}\) \cite{Hubara_2018_QRNN}; 3) sparsification which prunes unimportant connections or neurons, e.g., H-LSTM \cite{Dai_2020_PruneRNN}; and 4) tensor decomposition that transforms the weights to low rank tensor formats \cite{Novikov_2015_TT}. Among which, tensor decomposition performs significant superiority over other compression approaches when orienting to RNNs, since the extreme compression ratio, i.e., more than \(10^5 \times\) \cite{Yang_2017_TTRNN,Ye_2018_BTD,Pan_2019_TRRNN,Yin_2020_HTRNN}, could be obtained. Furthermore, these results based on tensor decomposition \cite{Yang_2017_TTRNN,Ye_2018_BTD,Pan_2019_TRRNN,Yin_2020_HTRNN} also inspire us that, tensor networks have demonstrated the superiority over traditional CANDECOMP/PARAFAC (CP) and Tucker decomposition due to the elimination of the curse of dimensionality \cite{Cichocki_2018_TensorNetworks}.

To be specific, Tjandra et al. \cite{Tjandra_2017_TTRNN1} first implemented tensor train (TT) decomposition \cite{Oseledets_2011_InventTT} to compress RNNs, since TT has a relatively simple and regular data structure. After that, Yang et al. \cite{Yang_2017_TTRNN} followed them to compress LSTM \& GRU and gained amazing compression ratio and accuracy on UCF11 \cite{Liu_2011_UCF11} and Youtube Celebrities Face \cite{Kim_2008_YCF} datasets. Ye et al. \cite{Ye_2018_BTD} used block term (BT) decomposition \cite{DeLathauwer_2008_InventBTD} to replace TT and obtained better accuracy on UCF11, while the amount of calculations of BT is more serious since its high-ordered Tucker kernel tensors. Pan et al. \cite{Pan_2019_TRRNN} selected tensor ring (TR) decomposition \cite{Zhao_2018_TR} to achieve higher compression ratio and accuracy on UCF11 as well, but more computation complexity is inevitable because of more core tensors and ranks. Most recently, Yin et al. \cite{Yin_2020_HTRNN} and Wu et al. \cite{Wu_2020_Hybrid} applied relatively complex hierarchical Tucker (HT) decomposition \cite{Grasedyck_2010_InventHT} and acquired better accuracy than TT on UCF11, however, the computation complexity of their practices is also intractable because HT has a kind of complex binary tree structure. Reviewing these practices and their conclusions, it seems that in the aspects of accuracy and space complexity, the selection priority of tensor decomposition method is TT \(<\) BT \(<\) TR \(<\) HT, i.e., HT is the best to gain the most top accuracy and the minimum number of parameters while TT is the worst. However, only TT appears to be efficient in the practical computations due to its efficient sequenced contractions \cite{Novikov_2015_TT}. Scilicet, none of these tensor decomposition could satisfy us in both space and computation complexity.

\begin{figure}
\centering
\includegraphics[width=0.35\textwidth]{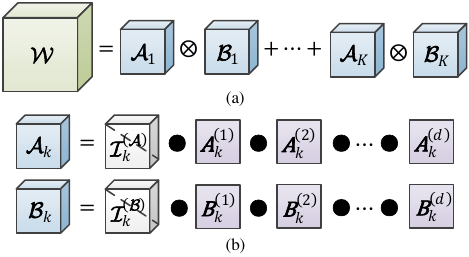}
\caption{The (a) KT decomposition described in Equation (\ref{Eq_general_ktd}) and (b) the derived KCP decomposition proposed in \cite{Phan_2013_KTD2}.}
\label{Fig_general_ktd}
\end{figure}

In this paper, we consider \emph{Kronecker CANDECOMP/PARAFAC (KCP)} decomposition \cite{Phan_2013_KTD2}, which is derived from another classical tensor decomposition method, i.e., \emph{Kronecker tensor (KT)} decomposition \cite{Phan_2012_KTD1}, to compress RNNs since its factor matrices are sparse and partitioned that might help to decrease both the space and computation complexity considerably. Generally, for a \(d\)th-order tensor \(\bm{\mathcal{W}} \in \mathbb{R} ^{l_1 \times l_2 \times \cdots \times l_d}\), if its every mode can be factorized as \(l_i=m_{i}n_{i}\) (\(i \in \{1,2,\cdots,d\}\)), then there exist \(K\) pairs of factor tensors \(\bm{\mathcal{A}}_{k} \in \mathbb{R} ^{m_1 \times m_2 \times \cdots \times m_d}\) and \(\bm{\mathcal{B}}_{k} \in \mathbb{R} ^{n_1 \times n_2 \times \cdots \times n_d}\) that yield the KT format of \(\bm{\mathcal{W}}\) like \cite{Phan_2012_KTD1}
\begin{equation}\label{Eq_general_ktd}
\bm{\mathcal{W}}=\sum_{k=1}^{K}{\bm{\mathcal{A}}_{k} \otimes \bm{\mathcal{B}}_{k}}
\end{equation}
where \(\otimes\) denotes the Kronecker or tensor product, and \(K\) is termed as \emph{KT rank}. Figure \ref{Fig_general_ktd}(a) illustrates the structure of KT decomposition vividly. The most important characteristic of KT is that it can be transformed into a new CP format, i.e., KCP, whose every factor matrix is actually a sparse partitioned matrix according to Lemma 2.1 in \cite{Phan_2013_KTD2}. Concretely, both \(\bm{\mathcal{A}}_{k}\) and \(\bm{\mathcal{B}}_{k}\) could be further decomposed into the CP format as drawn in Figure \ref{Fig_general_ktd}(b), then these produced factor matrices can be concatenated together as the new factor matrices of CP format of \(\bm{\mathcal{W}}\). Except the conventional space efficiency, we realize that this KCP format with sparse partitioned factor matrices is very friendly to the efficient computation even the parallel computing. Moreover, according to our experiments on UCF11, Youtube Celebrities Face, and UCF50 datasets, our KCP-RNNs have the comparable performance of accuracy with other tensor-decomposed RNNs, and lead to the preferable efficiency in both space and computation.

The major merits of this work are shown below:

\begin{itemize}
    \item We propose a novel KCP-RNN to compress traditional RNNs, and achieve competent performance compared with other published practices.
    \item We propose two algorithms based on several lemmas and theorems to accelerate the multiplication within KCP-RNNs.
    \item We discover that different tensor-decomposed RNNs have similar level in accuracy, but KCP-RNNs could take into account both the space and computation complexity.
\end{itemize}

The rest of this paper is organized as follows. Section \ref{sec:Relate} further discusses current tensor decomposition methods for compressing RNNs, emphasizes their limitation in computation. Section \ref{sec:KTDKCP} introduces and proves some basic laws of KCP decomposition. Section \ref{sec:KCPRNN} introduces our methods based on KCP format. Section \ref{sec:Exp} verifies the effect of our methods and discovers that different tensor-decomposed RNNs have no evident differences in the final recognition scores. Section \ref{sec:Disc} gives some further discussions about our methods and experiments. Section \ref{sec:Conc} concludes this work.

To maintain the consistency of mathematical notation, we will use the bold lower case letter as the vector symbol (e.g. \( \bm{a} \)), the bold upper case letter as the matrix symbol (e.g. \( \bm{A} \)), and the calligraphic bold upper case letter as the tensor notation (e.g. \( \bm{ \mathcal{A} } \)) in the following contents as well.
% \section{Introduction} end
%%%%%%%%%%%%%%%%%%%%%%%%%%%%%%%%%%%%%%%%%%%%%%%%%%%%%%%%%%%%%%%%%%%%%%%%%%%%%%%%%%%%%%%%%%%

%%%%%%%%%%%%%%%%%%%%%%%%%%%%%%%%%%%%%%%%%%%%%%%%%%%%%%%%%%%%%%%%%%%%%%%%%%%%%%%%%%%%%%%%%%%
\section{Related Works}\label{sec:Relate}

By utilizing TT, Yang et al. \cite{Yang_2017_TTRNN} first gained extremely amazing compression ratio and accuracy which verified the potential of tensor-decomposed neural networks. Specifically, for a weight matrix \(\bm{W} \in \mathbb{R} ^{M \times N}\) with \(M=\prod_{i=1}^{d} m_{i}\) and \(N=\prod_{i=1}^{d} n_{i}\), it can be reshaped into the \(d\)th-order tensor \(\bm{\mathcal{W}}\) which is as the same as that in Equation (\ref{Eq_general_ktd}), and we have \cite{Oseledets_2011_InventTT,Novikov_2015_TT,Yang_2017_TTRNN}
\begin{equation}\label{Eq_tt}
\bm{\mathcal{W}}=\bm{\mathcal{G}}_{1} \bullet \bm{\mathcal{G}}_{2} \bullet \cdots \bullet \bm{\mathcal{G}}_{d}
\end{equation}
where any single \( \bm{\mathcal{G}}_{i} \in \mathbb{R} ^{r_{i-1} \times m_{i}n_{i} \times r_{i}} \) is called the core tensor, all the \(r_i\) (\(r_0=r_d=1\)) are termed as TT ranks, and the operator \(\bullet\) is the contracted product or contraction. The corresponding input vector \(\bm{x}\) can also be tensorized into a \(d\)th-order tensor \(\bm{\mathcal{X}} \in \mathbb{R} ^{m_{1} \times m_{2} \times \cdots \times m_{d}}\), which can be contracted with \(\bm{\mathcal{G}}_{i}\) one by one so that the computation of this procedure is relatively efficient. Moreover, tensor network graph \cite{Espig_2011_TensorGraph} shown in Figure \ref{Fig_other_tng}(a) (limited in 4th-order for convenience) could express the structure of all the contractions distinctly, where each node denotes a single tensor, the connected lines represent the corresponding modes, and any line between two nodes could be contracted.

\begin{figure}
\centering
\includegraphics[width=0.45\textwidth]{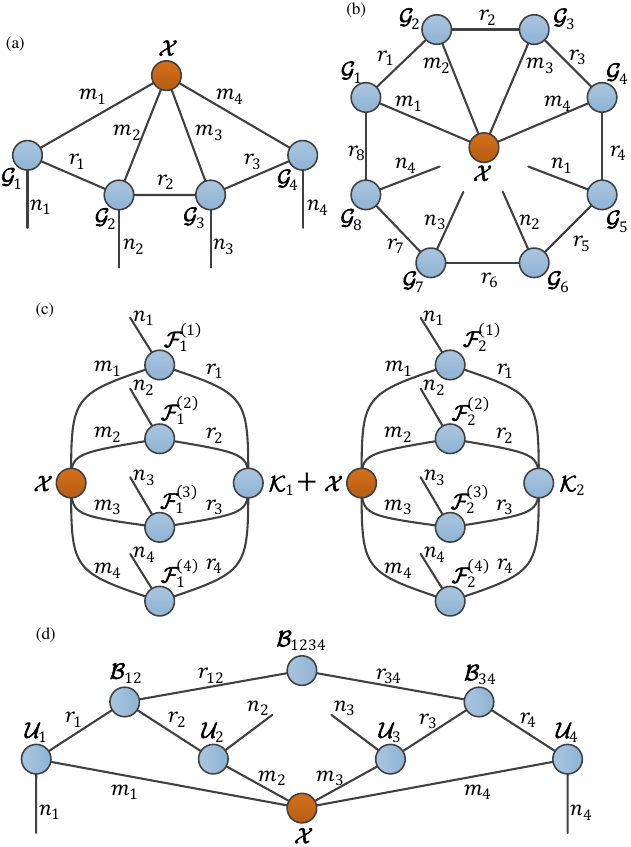}
\caption{Tensor network graphs of multiplication between input \(\bm{\mathcal{X}} \in \mathbb{R} ^{m_1 \times m_2 \times m_3 \times m_4}\) and weight \(\bm{\mathcal{W}} \in \mathbb{R} ^{m_{1}n_{1} \times m_{2}n_{2} \times m_{3}n_{3} \times m_{4}n_{4}}\) in (a) TT (b) TR (c) BT and (d) HT. Note that when the input \(\bm{\mathcal{X}}\) is contracted, \(m_i\) and \(n_i\) in \(\bm{\mathcal{W}}\) are two modes rather than a single mode \(m_{i}n_{i}\), so matrices \(\bm{F}_{p}^{(i)}\) and \(\bm{U}_{i}\) are drawn like tensors here. Besides, \(\bm{\mathcal{X}}\) can link to all the other nodes only in TT format.}
\label{Fig_other_tng}
\end{figure}

Pan et al. \cite{Pan_2019_TRRNN} utilized TR whose corresponding tensorizing approach appears to be the most special. For the weight tensor \(\bm{\mathcal{W}}\) defined in Equation (\ref{Eq_general_ktd}), TR generally regards it as the \(2d\)th-order rather than \(d\)th-order tensor as
\begin{equation}\label{Eq_tr}
\bm{\mathcal{W}}=\bm{\mathcal{G}}_{1} \bullet \bm{\mathcal{G}}_{2} \bullet \cdots \bullet \bm{\mathcal{G}}_{d} \bullet \bm{\mathcal{G}}_{d+1} \bullet \bm{\mathcal{G}}_{d+2} \bullet \cdots \bullet \bm{\mathcal{G}}_{d+d}
\end{equation}
where \( \bm{\mathcal{G}}_{i} \in \mathbb{R} ^{r_{i-1} \times m_{i} \times r_{i}} \) \& \( \bm{\mathcal{G}}_{d+i} \in \mathbb{R} ^{r_{d+i-1} \times n_{i} \times r_{d+i}} \) are core tensors, and all the \(r_{i}\) \& \(r_{d+i}\) are TR ranks (\(r_0=r_{2d} \neq 1\) hence two paired ranks should be contracted at the last operation). Compared with TT, it is clear that TR is more fine-grained so that higher compression ratio might be obtained. However, almost half of the core tensors can not be connected to the input \(\bm{\mathcal{X}}\) directly according to Figure \ref{Fig_other_tng}(b) so that the poor computation complexity is the impassable obstacle for TR.

Ye et al. \cite{Ye_2018_BTD} tried BT to replace TT to compress LSTMs and got better accuracy. For the same weight tensor \(\bm{\mathcal{W}}\) in Equation (\ref{Eq_general_ktd}), its BT format can be written as \cite{DeLathauwer_2008_InventBTD,Ye_2018_BTD}
\begin{equation}\label{Eq_btd}
\bm{\mathcal{W}}=\sum_{p=1}^{P} \bm{\mathcal{K}}_{p} \bullet \bm{F}_{p}^{(1)} \bullet \bm{F}_{p}^{(2)} \bullet \cdots \bullet \bm{F}_{p}^{(d)}
\end{equation}
where the right side of the equation is in fact a sum of \(P\) Tucker-decomposed tensors, \(\bm{\mathcal{K}}_{p} \in \mathbb{R} ^{r_1 \times r_2 \times \cdots \times r_d}\) is the kernel tensor of the \(p\)th Tucker, and \(\bm{F}_{p}^{(i)} \in \mathbb{R} ^{r_{i} \times m_{i}n_{i}}\) means the \(i\)th factor matrix of the \(p\)th Tucker. It is obvious that more computations should be considered during contracting with \(\bm{\mathcal{X}}\) according to Figure \ref{Fig_other_tng}(c), where \(\bm{\mathcal{X}}\) has no direct connections to the kernel tensors \(\bm{\mathcal{K}}_{p}\).

HT is relatively the most complex tensor-decomposed format as there are multiple different approaches to split the modes of \(\bm{\mathcal{W}}\) in Equation (\ref{Eq_general_ktd}) \cite{Grasedyck_2010_InventHT}. Generally, both Yin et al. \cite{Yin_2020_HTRNN} and Wu et al. \cite{Wu_2020_Hybrid} selected the balanced binary tree to be the basis of HT format for the sake of simplicity, and the specific format is \cite{Wu_2020_Hybrid}
\begin{equation}\label{Eq_ht}
\begin{aligned}
\bm{\mathcal{W}}=&\bm{U}_{1} \bullet \bm{\mathcal{B}}_{12} \bullet \bm{U}_{2} \bullet \bm{\mathcal{B}}_{1234} \bullet \cdots \\
&\bullet \bm{U}_{d/2} \bullet \bm{\mathcal{B}}_{12 \cdots d} \bullet \bm{U}_{d/2 + 1} \bullet \cdots \bullet \bm{U}_{d}
\end{aligned}
\end{equation}
where each \(\bm{U}_{i} \in \mathbb{R} ^{m_{i}n_{i} \times r_{i}}\) is called truncated matrix, \(\bm{\mathcal{B}}_{xy} \in \mathbb{R} ^{r_{x} \times r_{y} \times r_{xy}}\) (\(x \in \{1,3,\cdots,d-1,12,56,\cdots\}\) \& \(y \in \{2,4,\cdots,d,34,78,\cdots\}\)) is termed as transfer tensor, and all these \(r_{i}\), \(r_{x}\), \(r_{y}\) and \(r_{xy}\) are HT ranks. Beyond all doubt, even the balanced HT is still complicated, and Figure \ref{Fig_other_tng}(d) expressly illustrates the complex topology of the multiplication between \(\bm{\mathcal{X}}\) and \(\bm{\mathcal{W}}\), where all the transfer tensors are not connected to \(\bm{\mathcal{X}}\) so the most efficient inorder traversal can hardly surpass the computation speed of TT \cite{Wu_2020_Hybrid}.

In a word, except TT, all the other tensor decomposition methods can not afford us a satisfying computation efficiency, while conversely TT is the most coarse-grained format which is difficult to provide an impressive compression ratio. On the other hand, there are no obvious variations of precision among different decomposition methods in theory \cite{Cichocki_2016_TensorBook}, which have not been verified by the related works \cite{Yang_2017_TTRNN,Ye_2018_BTD,Pan_2019_TRRNN,Yin_2020_HTRNN,Wu_2020_Hybrid} since their experimental details are not exactly the same. Hence, it is very significant to develop a novel tensor format to optimize both the space and computation complexity of RNNs, and study whether the accuracy of different tensor-decomposed RNNs have differences.
% \section{Related Works} end
%%%%%%%%%%%%%%%%%%%%%%%%%%%%%%%%%%%%%%%%%%%%%%%%%%%%%%%%%%%%%%%%%%%%%%%%%%%%%%%%%%%%%%%%%%%

%%%%%%%%%%%%%%%%%%%%%%%%%%%%%%%%%%%%%%%%%%%%%%%%%%%%%%%%%%%%%%%%%%%%%%%%%%%%%%%%%%%%%%%%%%%
\section{KCP Decomposition}\label{sec:KTDKCP}

In this section, we introduce the knowledge of KCP to reinforce the theoretical foundation of our methods, by first analysing some characteristics of KT, since the latter is the seed of KCP. Furthermore, as it is hard to find proofs for some important conclusions in the studies of predecessors \cite{Phan_2012_KTD1,Phan_2013_KTD2}, we give our own ones here.

\subsection{Restricted KT Decomposition}

In fact, the most general KT format defined in Equation (\ref{Eq_general_ktd}) is very flexible, since \(l_i=m_{i}n_{i}\) might not be the unique factorization. For example, assume there is \(m_{i}=a_{i}b_{i}\), we can force \(l_i=a_{i} \times (b_{i}n_{i})\) and apply this factorization to some parts of factor tensors, then \(\bm{\mathcal{A}}_{k}\) and \(\bm{\mathcal{B}}_{k}\) under the different value of \(k\) will have different shapes. Undoubtedly, this flexibility is unfriendly to practical applications. Therefore, we just consider the KT format whose \(m_{i}\) and \(n_{i}\) are unchanged for all \(k\). We term this format as the \emph{restricted KT decomposition (RKT)}, and all the following descriptions of KT and KCP shall be based on RKT unless otherwise indicated.

For RNNs, weight matrices are in fact \(2\)nd-order tensors, thus researchers always reshaped them into high-ordered tensors to achieve better compression ratio just like Equation (\ref{Eq_tt}), and this approach is known as \emph{tensorizing} \cite{Novikov_2015_TT}. Obviously, in analogy with Equation (\ref{Eq_general_ktd}), KT format inherently has the pathway of tensorizing. Therefore, for the weight matrix \(\bm{W} \in \mathbb{R} ^{M \times N}\) defined in Equation (\ref{Eq_tt}), its KT format of tensorizing defined in Equation (\ref{Eq_general_ktd}) can yield a rank-\(K\) factorization of \(\bm{W}\) like \cite{Phan_2012_KTD1}
\begin{equation}\label{Eq_general_km}
\bm{W}=\sum_{k=1}^{K}{\bm{a}_{k} \bm{b}_{k}^{\rm T}}.
\end{equation}
where \(\bm{a}_{k}={\rm vec}(\bm{\mathcal{A}}_{k}) \in \mathbb{R} ^{M}\) and \(\bm{b}_{k}={\rm vec}(\bm{\mathcal{B}}_{k}) \in \mathbb{R} ^{N}\) are the vectorization of factor tensors. However, this equation defined in \cite{Phan_2012_KTD1} is lack of proof so we will give our own here with the definition of multi-indices which is also helpful to the following other proofs.

\begin{definition}[Multi-indices]\label{def_MultiInd}
For a \(d\)th-order tensor \(\bm{\mathcal{A}} \in \mathbb{R}^{n_1 \times n_2 \times \cdots \times n_d}\) whose entry has the indices \((\nu_{1}, \nu_{2}, \cdots, \nu_{d})\), if its \(k\) (\(k \in \{2,3,\cdots,d\}\)) modes \(n_{i+1}\), \(n_{i+2}\), \(\cdots\), \(n_{i+k}\) (\(i \in \{0,1,\cdots,d-2\}\)) are reshaped into a single mode \(m\), \textit{i.e.}, \(m = n_{i+1}n_{i+2} \cdots n_{i+k}\), the corresponding indices \((\nu_{i+1}, \nu_{i+2}, \cdots \nu_{i+k})\) will become \(\mu = \overline{\nu_{i+1} \nu_{i+2} \cdots \nu_{i+k}}\) and \(\mu \in \{0,1,\cdots,m-1\}\), which is termed as multi-indices and can be calculated as \cite{Dolgov_2014_MultiIndex}
\begin{equation}
\begin{aligned}
\mu &= \overline{\nu_{i+1} \nu_{i+2} \cdots \nu_{i+k}} \\ 
&= \nu_{i+1} + \nu_{i+2}n_{i+1} + \cdots + \nu_{i+k}n_{i+1} \cdots n_{i+k-1}
\end{aligned}
\end{equation}
in which we rule the value of index to begin at 0 for brevity.
\end{definition}

\begin{lemma}[Rank-\(K\) Factorization]\label{lem_rankk}
If the \(d\)th-order tensor \(\bm{\mathcal{W}} \in \mathbb{R} ^{m_{1}n_{1} \times m_{2}n_{2} \times \cdots \times m_{d}n_{d}}\) has the KT format defined in Equation (\ref{Eq_general_ktd}), then its \((M \times N)\) matricization can be represented by Equation (\ref{Eq_general_km}) with the constraints \(M=\prod_{i=1}^{d} m_{i}\) and \(N=\prod_{i=1}^{d} n_{i}\).
\end{lemma}

\begin{proof}
According to Equation (\ref{Eq_general_ktd}), the entry of \(\bm{\mathcal{W}}\) is
\begin{equation}\label{Eq_lem_1}
\begin{aligned}
&\mathcal{W}(\overline{\mu_{1}\nu_{1}},\overline{\mu_{2}\nu_{2}},\cdots,\overline{\mu_{d}\nu_{d}}) \\ 
= &\sum_{k=1}^{K}{\mathcal{A}_{k}(\mu_{1},\mu_{2},\cdots,\mu_{d})\mathcal{B}_{k}(\nu_{1},\nu_{2},\cdots,\nu_{d})}
\end{aligned}
\end{equation}
where \(\mu_{i} \in \{0,1,\cdots,m_{i}-1\}\) and \(\nu_{i} \in \{0,1,\cdots,n_{i}-1\}\). When we reshape \(\bm{\mathcal{W}}\) to the matrix \(\bm{W}\), because of the constraints \(M=\prod_{i=1}^{d} m_{i}\) and \(N=\prod_{i=1}^{d} n_{i}\), we have
\begin{equation}\label{Eq_lem_2}
\mathcal{W}(\overline{\mu_{1}\nu_{1}},\cdots,\overline{\mu_{d}\nu_{d}}) = W(\overline{\mu_{1}\cdots\mu_{d}},\overline{\nu_{1}\cdots\nu_{d}}). 
\end{equation}
Similarly, if we stretch \(\bm{\mathcal{A}}_{k}\) and \(\bm{\mathcal{B}}_{k}\) into vectors \(\bm{a}_{k}\) and \(\bm{b}_{k}\) respectively, we have
\begin{equation}\label{Eq_lem_3}
\begin{aligned}
&\mathcal{A}_{k}(\mu_{1},\mu_{2},\cdots,\mu_{d}) = a_{k}(\overline{\mu_{1}\mu_{2}\cdots\mu_{d}}) \\
&\mathcal{B}_{k}(\nu_{1},\nu_{2},\cdots,\nu_{d}) = b_{k}(\overline{\nu_{1}\nu_{2}\cdots\nu_{d}}).
\end{aligned}
\end{equation}
Based on the simultaneous Equations (\ref{Eq_lem_1}), (\ref{Eq_lem_2}) and (\ref{Eq_lem_3}), there is
\begin{equation}
\begin{aligned}
&W(\overline{\mu_{1}\mu_{2}\cdots\mu_{d}},\overline{\nu_{1}\nu_{2}\cdots\nu_{d}}) \\ = &\sum_{k=1}^{K}{a_{k}(\overline{\mu_{1}\mu_{2}\cdots\mu_{d}})b_{k}(\overline{\nu_{1}\nu_{2}\cdots\nu_{d}})}
\end{aligned}
\end{equation}
which can conclude this lemma.
\end{proof}

The Lemma \ref{lem_rankk} can afford us at least two aspects of inspirations. One is that KT contains tensorizing approach which is necessary to compress weight matrices in terms of tensors. The other is that the KT rank \(K\) might determine the precision of decomposition to a great extent.

\subsection{Bounds of KT Rank}

Though the KT rank \(K\) of weight matrix \(\bm{W}\) is important, finding a specific \(K\) is in fact an NP-hard problem since Equation (\ref{Eq_general_km}) appears to be a 2nd-order CP format. The following lemma can give the explanation.

\begin{lemma}[Solve \(K\) of \(\bm{W}\) is NP-hard]\label{lem_nphard}
For a known matrix \(\bm{W}\) defined in Equation (\ref{Eq_general_km}), to determine its KT rank \(K\) is NP-hard.
\end{lemma}

\begin{proof}
For any paired vectors, their spanned matrix is equivalent to their outer product, i.e.,
\begin{equation}
\bm{a}\bm{b}^{\rm T} = \bm{a} \circ \bm{b}
\end{equation}
where \(\circ\) denotes the outer product. Therefore, each \(\bm{a}_{k} \bm{b}_{k}^{\rm T}\) is a rank-1 \(2\)nd-order tensor (matrix), and Equation (\ref{Eq_general_km}) is a standard CP decomposition of \(\bm{W}\). As solving the CP rank is NP-hard \cite{Hillar_2013_NPHard}, this lemma is then proved.
\end{proof}

Even so, we can still work out the bounds of KT rank to help to design our following KCP-RNNs. Inspired by the external Lemma 1 in \cite{Khrulkov_2018_ExpPowerRNN}, we give a lemma bellow to discuss the bounds of \(K\).

\begin{lemma}[Bounds of KT Rank]\label{lem_KTRank}
For a known matrix \(\bm{W}\) defined in Equation (\ref{Eq_general_km}), if it has the KT format defined in Equation (\ref{Eq_general_ktd}), then the bounds of KT rank should be
\begin{equation}
R \leq K \leq C^{(\widetilde{\bm{\mathcal{W}}})}
\end{equation}
where \(R\) is the matrix rank of \(\bm{W}\), \(C^{(\widetilde{\bm{\mathcal{W}}})}\) is the CP rank of the (\(2d\))th-order tensor \(\widetilde{\bm{\mathcal{W}}} \in \mathbb{R} ^{m_{1} \times m_{2} \times \cdots \times m_{d} \times n_{1} \times n_{2} \times \cdots \times n_{d}}\).
\end{lemma}

\begin{proof}
Firstly, there is
\begin{equation}\label{Eq_thm_1_1}
R = {\rm rank}(\bm{W}) = {\rm rank}\left(\sum_{k=1}^{K}{\bm{a}_{k} \bm{b}_{k}^{\rm T}}\right).
\end{equation}
Secondly, for the matrices \(\bm{a}_{k} \bm{b}_{k}^{\rm T}\), we have
\begin{equation}\label{Eq_thm_1_2}
\begin{aligned}
&{\rm rank}\left(\sum_{k=1}^{K}{\bm{a}_{k} \bm{b}_{k}^{\rm T}}\right) \leq {\rm rank}(\bm{a}_{1} \bm{b}_{1}^{\rm T}) + {\rm rank}(\bm{a}_{2} \bm{b}_{2}^{\rm T}) + \cdots \\
&+ {\rm rank}(\bm{a}_{K} \bm{b}_{K}^{\rm T}) = K.
\end{aligned}
\end{equation}
Thirdly, in terms of Lemma \ref{lem_nphard}, \(K\) is the CP rank of \(\bm{W}\), and \(\bm{W}\) is in fact the matricization of \(\widetilde{\bm{\mathcal{W}}}\) rather than \(\bm{\mathcal{W}}\). According to Lemma 1 in \cite{Khrulkov_2018_ExpPowerRNN}, the CP rank of a tensor is always not less than the rank of its matricization. Therefore, we have
\begin{equation}\label{Eq_thm_1_3}
K \leq C^{(\widetilde{\bm{\mathcal{W}}})}.
\end{equation}
Finally, this lemma is established by combing (\ref{Eq_thm_1_1}), (\ref{Eq_thm_1_2}) and (\ref{Eq_thm_1_3}).
\end{proof}

Trace back to the BT decomposition defined in Equation (\ref{Eq_btd}), there should be \(1 \leq P < C^{(\bm{\mathcal{W}})}\), which hints that the KT rank \(K\) could be higher than the BT rank \(P\) since usually there are \(1 < R\) and \(C^{(\bm{\mathcal{W}})} < C^{(\widetilde{\bm{\mathcal{W}}})}\). This might imply that KT has more potential ability of parallel computing than BT.

\subsection{From KT to KCP}

Suppose the maximum \(m_{i}\) and \(n_{i}\) are \(m\) and \(n\) respectively, it is easy to conclude that the space complexity of Equation (\ref{Eq_general_km}) should be
\begin{equation}\label{Eq_space_comp_1}
\mathcal{O}\left({(m^d+n^d)K}\right)
\end{equation}
which is not very efficient since it still has exponential terms. Fortunately, Lemma 2.1 in \cite{Phan_2013_KTD2} claims that if \(\bm{a}_{k}={\rm vec}(\bm{\mathcal{A}}_{k}) \in \mathbb{R} ^{M}\) and \(\bm{b}_{k}={\rm vec}(\bm{\mathcal{B}}_{k}) \in \mathbb{R} ^{N}\) could still be decomposed as the \(d\)th-order tensors in CP format, \(\bm{\mathcal{W}}\) will become a sparse CP format, i.e., KCP, which can make the space complexity lower. However, the proof of this lemma is still lacked in \cite{Phan_2013_KTD2}, so we give our own proof below.

\begin{theorem}[KCP]\label{Thm_KCP}
For the tensor \(\bm{\mathcal{W}}\) in KT format defined in Equation (\ref{Eq_general_ktd}), if its factor tensors \(\bm{\mathcal{A}}_{k}\) and \(\bm{\mathcal{B}}_{k}\) could be further decomposed in CP format like
\begin{equation}\label{Eq_kt_to_kcp}
\begin{aligned}
&\bm{\mathcal{A}}_{k} = \bm{\mathcal{I}}_{k}^{(\bm{\mathcal{A}})} \bullet \bm{A}_{k}^{(1)} \bullet \bm{A}_{k}^{(2)} \bullet \cdots \bullet \bm{A}_{k}^{(d)} \\
&\bm{\mathcal{B}}_{k} = \bm{\mathcal{I}}_{k}^{(\bm{\mathcal{B}})} \bullet \bm{B}_{k}^{(1)} \bullet \bm{B}_{k}^{(2)} \bullet \cdots \bullet \bm{B}_{k}^{(d)}
\end{aligned}
\end{equation}
where \(\bm{A}_{k}^{(i)} \in \mathbb{R} ^{m_{i} \times C_{k}^{(\bm{\mathcal{A}})}}\), \(\bm{B}_{k}^{(i)} \in \mathbb{R} ^{n_{i} \times C_{k}^{(\bm{\mathcal{B}})}}\), both \(\bm{\mathcal{I}}_{k}^{(\bm{\mathcal{A}})}\) and \(\bm{\mathcal{I}}_{k}^{(\bm{\mathcal{B}})}\) are superdiagonal tensors, \(C_{k}^{(\bm{\mathcal{A}})}\) and \(C_{k}^{(\bm{\mathcal{B}})}\) are the CP ranks of \(\bm{\mathcal{A}}_{k}\) and \(\bm{\mathcal{B}}_{k}\) respectively. Then \(\bm{\mathcal{W}}\) could also be represented in CP format as
\begin{equation}\label{Eq_kcp}
\bm{\mathcal{W}} = \bm{\mathcal{I}}^{(\bm{\mathcal{W}})} \bullet \bm{W}^{(1)} \bullet \bm{W}^{(2)} \bullet \cdots \bullet \bm{W}^{(d)}
\end{equation}
where \(\bm{W}^{(i)} \in \mathbb{R} ^{m_{i}n_{i} \times C^{(\bm{\mathcal{W}}})}\), and \(C^{(\bm{\mathcal{W}})}=\sum_{k=1}^{K}{C_{k}^{(\bm{\mathcal{A}})}C_{k}^{(\bm{\mathcal{B}})}}\) is the CP rank of \(\bm{\mathcal{W}}\). Besides, each \(\bm{W}^{(i)}\) is a sparse partitioned matrix like
\begin{equation}\label{Eq_W_i}
\bm{W}^{(i)} = \begin{bmatrix} \bm{A}_{1}^{(i)} \otimes \bm{B}_{1}^{(i)} & \bm{A}_{2}^{(i)} \otimes \bm{B}_{2}^{(i)} & \cdots & \bm{A}_{K}^{(i)} \otimes \bm{B}_{K}^{(i)} \end{bmatrix}.
\end{equation}
\end{theorem}

\begin{proof}
The proof can be found in Appendix.
\end{proof}

\subsection{Space Complexity and Compression Ratio}

Although CP is negative to deal with the curse of dimensionality \cite{Cichocki_2018_TensorNetworks}, the KCP format defined in Theorem \ref{Thm_KCP} has sparse and partitioned factor matrices because of the existence of Kronecker products, so the complexity of the superdiagonal kernel tensor \(\bm{\mathcal{I}}^{(\bm{\mathcal{W}})}\) might be roughly ignored by assuming the nonzero entries from \(\bm{\mathcal{I}}^{(\bm{\mathcal{W}})}\) can construct the vector \(\bm{1}_{C^{(\bm{\mathcal{W}}})}\) \cite{Phan_2020_CP1Vector}. If we suppose the maximum CP ranks of \(C_{k}^{(\bm{\mathcal{A}})}\) and \(C_{k}^{(\bm{\mathcal{B}})}\) are \(C^{(\bm{\mathcal{A}})}\) and \(C^{(\bm{\mathcal{B}})}\) respectively, the space complexity of Equation (\ref{Eq_kcp}) is easy to be obtained as
\begin{equation}\label{Eq_space_comp_2}
\mathcal{O}\left(d(mC^{(\bm{\mathcal{A}})}+nC^{(\bm{\mathcal{B}})})K\right)
\end{equation}
which no longer includes exponential terms compared with Equation (\ref{Eq_space_comp_1}). Naturally, the compression ratio of KCP should be
\begin{equation}\label{Eq_comp_ratio}
\Phi = \frac{(mn)^{d}}{d(mC^{(\bm{\mathcal{A}})}+nC^{(\bm{\mathcal{B}})})K}
\end{equation}
which evidently reduces the space complexity of \(\bm{\mathcal{W}}\) from the exponential to the polynomial level.
% \section{KT and KCP Decomposition} end
%%%%%%%%%%%%%%%%%%%%%%%%%%%%%%%%%%%%%%%%%%%%%%%%%%%%%%%%%%%%%%%%%%%%%%%%%%%%%%%%%%%%%%%%%%%

%%%%%%%%%%%%%%%%%%%%%%%%%%%%%%%%%%%%%%%%%%%%%%%%%%%%%%%%%%%%%%%%%%%%%%%%%%%%%%%%%%%%%%%%%%%
\section{RNNs in KCP Format}\label{sec:KCPRNN}

In this section, we first give two algorithms to deal with the multiplication between the input tensor \(\bm{\mathcal{X}}\) and the weight tensor \(\bm{\mathcal{W}}\) in KCP format, then discuss the superiority of KCP in the aspect of complexity, finally propose the KCP-RNNs based on LSTMs.

\begin{figure*}
\centering
\includegraphics[width=0.99\textwidth]{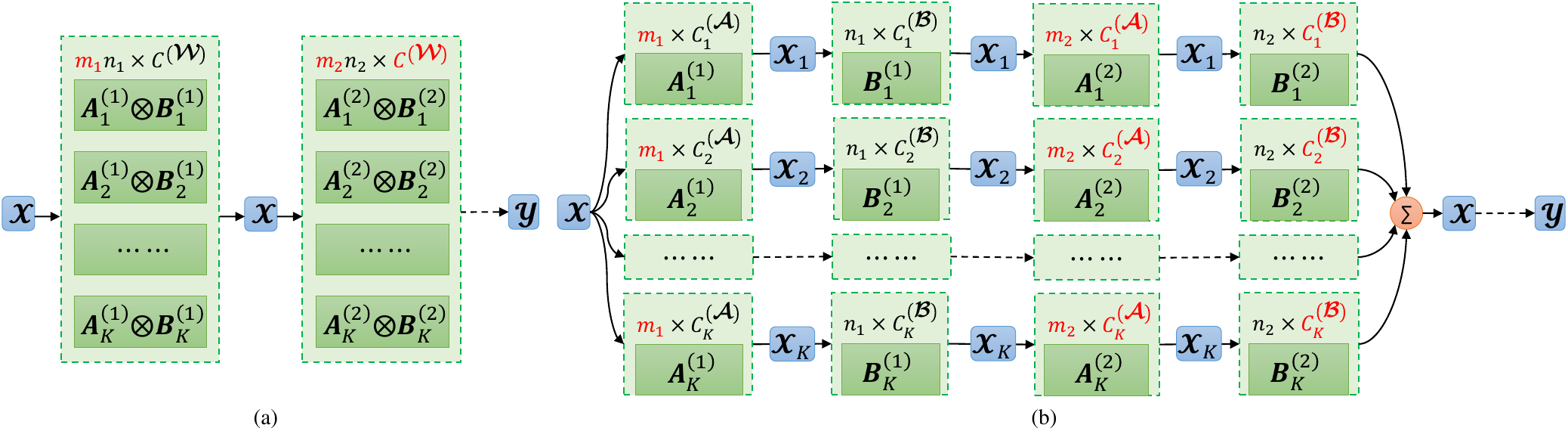}
\caption{Multiplication between input and weight in KCP format according to (a) Algorithm \ref{Alg_KCP} and (b) Algorithm \ref{Alg_KCP_2}. The modes drawn in red should be contracted.}
\label{Fig_multiplication}
\end{figure*}

\subsection{Fast Multiplication between Input and Weight}

Since the normal vector-matrix multiplication is also contraction, the input tensor can compute with the factor matrices defined in Equation (\ref{Eq_kcp}) one by one, and inspired by \cite{Ye_2018_BTD}, making the kernel tensor as the last contracted one is better than coping its original location. Thus the multiplication between \(\bm{\mathcal{X}}\) and \(\bm{\mathcal{W}}\) in KCP format can be represented like
\begin{equation}\label{Eq_multiplication}
\bm{\mathcal{Y}} = \bm{\mathcal{X}} \bullet \bm{W}^{(1)} \bullet \bm{W}^{(2)} \bullet \cdots \bullet \bm{W}^{(d)} \bullet \bm{\mathcal{I}}^{(\bm{\mathcal{W}})}
\end{equation}
where \(\bm{\mathcal{Y}} \in \mathbb{R} ^{n_{1} \times n_{2} \times \cdots \times n_{d}}\) is the output tensor that could be reshaped into \(\bm{y} \in \mathbb{R} ^{N}\). However, with the continuous contractions from left to right in Equation (\ref{Eq_multiplication}), every mode \(C^{(\bm{\mathcal{W}})}\) is accumulated as only \(m_{i}\) is contracted each time, i.e., the mode \(\left(C^{(\bm{\mathcal{W}})}\right)^{d}\) shall be contracted at the last \(\bullet\) in Equation (\ref{Eq_multiplication}), and the corresponding computation complexity should be
\begin{equation}\label{Eq_compcomplex_naive}
\mathcal{O}\left((d{\rm max}\{m,n\}^{d+1}+n^d)\left(C^{(\bm{\mathcal{A}})}C^{(\bm{\mathcal{B}})}K\right)^{d}\right)
\end{equation}
which increases exponentially with the ranks.

Fortunately, the kernel tensor of CP is superdiagonal while that of Tucker is not, so we could contract every \(C^{(\bm{\mathcal{W}})}\) gradually during the procedure of Equation (\ref{Eq_multiplication}) without remarkable influence. Algorithm \ref{Alg_KCP} describes our fast multiplication in detail, and Figure \ref{Fig_multiplication}(a) illustrates this strategy vividly. Note that the kernel tensor \(\bm{\mathcal{I}}^{(\bm{\mathcal{W}})}\) is resolved into each factor matrix, so the curse of dimensionality is solved to some extent.

\begin{algorithm}[htb] 
\caption{Strict Fast Multiplication between Input and Weight in KCP Format.}
\label{Alg_KCP} 
\begin{algorithmic}[1]
\REQUIRE ~~\\
Input data \(\bm{\mathcal{X}} \in \mathbb{R} ^{m_{1} \times m_{2} \times \cdots \times m_{d}}\);\\
KCP-weight \(\bm{\mathcal{W}}\) defined in Equation (\ref{Eq_kcp}).\\
\ENSURE ~~\\
Output data \(\bm{\mathcal{Y}} \in \mathbb{R} ^{n_{1} \times n_{2} \times \cdots \times n_{d}}\) in Equation (\ref{Eq_multiplication}).
\FOR{\(i = 1 \longrightarrow d\)}
\FOR{\(k = 1 \longrightarrow K\)}
\STATE \(\bm{D}_{k}^{(i)} \in \mathbb{R} ^{m_{i}n_{i} \times C_{k}^{(\bm{\mathcal{A}})}C_{k}^{(\bm{\mathcal{B}})}} \longleftarrow \bm{A}_{k}^{(i)} \otimes \bm{B}_{k}^{(i)}\);
\ENDFOR
\STATE \(\bm{W}^{(i)} \in \mathbb{R} ^{m_{i}n_{i} \times C^{(\bm{\mathcal{W}})}} \longleftarrow\) merge all \(\bm{D}_{k}^{(i)}\) according to Equation (\ref{Eq_W_i});
\IF{\(i\) mod \(2 \neq 0\)}
\STATE \(\bm{\mathcal{X}}^{(i)} \in \mathbb{R} ^{n_{1} \cdots n_{i-1}m_{i+1} \cdots m_{d} \times m_{i}} \longleftarrow\) reshape \(\bm{\mathcal{X}}\);
\STATE \(\bm{W}^{(i)} \in \mathbb{R} ^{m_{i} \times n_{i}C^{(\bm{\mathcal{W}})}} \longleftarrow\) reshape \(\bm{W}^{(i)}\);
\STATE \(\bm{\mathcal{X}} \in \mathbb{R} ^{n_{1} \cdots n_{i-1}m_{i+1} \cdots m_{d} \times n_{i}C^{(\bm{\mathcal{W}})}} \longleftarrow \bm{\mathcal{X}}^{(i)} \bullet \bm{W}^{(i)}\);
\ELSIF{\(i\) mod \(2 = 0\)}
\STATE \(\bm{\mathcal{X}}^{(i)} \in \mathbb{R} ^{n_{1} \cdots n_{i-1}m_{i+1} \cdots m_{d} \times m_{i}C^{(\bm{\mathcal{W}})}} \longleftarrow\) reshape \(\bm{\mathcal{X}}\);
\STATE \(\bm{W}^{(i)} \in \mathbb{R} ^{m_{i}C^{(\bm{\mathcal{W}})} \times n_{i}} \longleftarrow\) reshape \(\bm{W}^{(i)}\);
\STATE \(\bm{\mathcal{X}} \in \mathbb{R} ^{n_{1} \cdots n_{i-1}m_{i+1} \cdots m_{d} \times n_{i}} \longleftarrow \bm{\mathcal{X}}^{(i)} \bullet \bm{W}^{(i)}\);
\ENDIF
\ENDFOR
\IF{\(d\) mode \(2 \neq 0\)}
\STATE \(\bm{\mathcal{Y}} \longleftarrow \bm{\mathcal{X}} \bullet \bm{i}^{(\bm{\mathcal{W}})}\), vector \(\bm{i}^{(\bm{\mathcal{W}})} \in \mathbb{R} ^{C^{(\bm{\mathcal{W}})}}\) from \(\bm{\mathcal{I}}^{(\bm{\mathcal{W}})}\) should be contracted to remove the mode \(C^{(\bm{\mathcal{W}})}\);
\ELSE
\STATE \(\bm{\mathcal{Y}} \longleftarrow \bm{\mathcal{X}}\);
\ENDIF
\RETURN \(\bm{\mathcal{Y}}\).
\end{algorithmic}
\end{algorithm}

Nevertheless, in theory, computation complexity of Algorithm \ref{Alg_KCP} is not efficient enough, because the considerable accumulation of CP ranks \(C^{(\bm{\mathcal{W}})}=\sum_{k=1}^{K}{C_{k}^{(\bm{\mathcal{A}})}C_{k}^{(\bm{\mathcal{B}})}}\) might make the size of \(\bm{W}^{(i)}\) very large. Therefore, a more relaxed strategy, which breaks the barrier between contraction \(\bullet\) and Kronecker product \(\otimes\) to let the input \(\bm{\mathcal{X}}\) pass each \(\bm{A}_{k}^{(i)}\) and \(\bm{B}_{k}^{(i)}\) in sequence, should be taken into account. In detail, for \(\bm{\mathcal{Y}} = \bm{\mathcal{X}} \bullet \bm{W}^{(1)} \bullet \bm{W}^{(2)}\) which shall be reprocessed to split the CP rank \(C^{(\bm{\mathcal{W}})}\), each \(\bm{A}_{k}^{(i)}\) and \(\bm{B}_{k}^{(i)}\) should be multiplied separately. This procedure could ensure the output \(\bm{\mathcal{Y}}\) invariant according to the theorem below.

\begin{theorem}[Relaxed Fast Multiplication]\label{Thm_Multiply}
For the input \(\bm{\mathcal{X}} \in \mathbb{R} ^{m_1 \times m_2}\) and the output \(\bm{\mathcal{Y}} = \bm{\mathcal{X}} \bullet \bm{W}^{(1)} \bullet \bm{W}^{(2)}\) which is dealt with by Algorithm \ref{Alg_KCP}, \(\bm{\mathcal{Y}}\) could also be obtained as
\begin{equation}\label{Eq_RelaxFast}
\begin{aligned}
& \bm{\mathcal{X}}_{k} = \left(\left(\left(\left(\bm{\mathcal{X}} \bullet \bm{A}_{k}^{(1)}\right) \otimes \bm{B}_{k}^{(1)} \right) \bullet {\rm vec}(\bm{A}_{k}^{(2)}) \right) \bullet \left(\bm{B}_{k}^{(2)}\right)^{\rm T}\right)\\
& \bm{\mathcal{Y}} = \sum _{k} {\bm{\mathcal{X}}_{k}}.
\end{aligned}
\end{equation}
\end{theorem}

\begin{proof}
The proof can be found in Appendix.
\end{proof}

It should be emphasized that Theorem \ref{Thm_Multiply} can only deal with \(\bm{\mathcal{X}} \bullet \bm{W}^{(1)} \bullet \bm{W}^{(2)}\). In other words, it cannot be naively extended to \(\bm{\mathcal{X}} \bullet \bm{W}^{(1)} \bullet \bm{W}^{(2)} \bullet \bm{W}^{(3)} \bullet \bm{W}^{(4)}\) and any longer case that \(d > 4\), because there is no inherent contraction between \(\bm{W}^{(2)}\) and \(\bm{W}^{(3)}\) as shown in Figure \ref{Fig_mul_tng}(b). Accordingly, \(\sum _{k} {\bm{\mathcal{X}}_{k}}\) must be done when every even number of \(i\) is passed.

\begin{algorithm}[htb] 
\caption{Relaxed Fast Multiplication between Input and Weight in KCP Format (\(d\) is even).}
\label{Alg_KCP_2} 
\begin{algorithmic}[1]
\REQUIRE ~~\\
Input data \(\bm{\mathcal{X}} \in \mathbb{R} ^{m_{1} \times m_{2} \times \cdots \times m_{d}}\);\\
KCP-weight \(\bm{\mathcal{W}}\) defined in Equation (\ref{Eq_kcp}).\\
\ENSURE ~~\\
Output data \(\bm{\mathcal{Y}} \in \mathbb{R} ^{n_{1} \times n_{2} \times \cdots \times n_{d}}\) in Equation (\ref{Eq_multiplication}).
\FOR{\(i = 1 \longrightarrow d\)}
\FOR{\(k = 1 \longrightarrow K\)}
\IF{\(i\) mod \(2 \neq 0\)}
\STATE \(\bm{\mathcal{X}}_{k} \longleftarrow \bm{\mathcal{X}}\)
\STATE \(\bm{\mathcal{X}}_{k} \in \mathbb{R} ^{n_{1} \cdots n_{i-1}m_{i+1} \cdots m_{d} \times m_{i}} \longleftarrow\) reshape \(\bm{\mathcal{X}}_{k}\);
\STATE \(\bm{\mathcal{X}}_{k} \in \mathbb{R} ^{n_{1} \cdots n_{i-1}m_{i+1} \cdots m_{d} \times C_{k}^{(\bm{\mathcal{A}})}} \longleftarrow \bm{\mathcal{X}}_{k} \bullet \bm{A}_{k}^{(i)}\);
\STATE \(\bm{\mathcal{X}}_{k} \in \mathbb{R} ^{n_{1} \cdots n_{i-1}m_{i+1} \cdots m_{d}n_{i} \times C_{k}^{(\bm{\mathcal{A}})}C_{k}^{(\bm{\mathcal{B}})}} \longleftarrow \bm{\mathcal{X}}_{k} \otimes \bm{B}_{k}^{(i)}\);
\ELSIF{\(i\) mod \(2 = 0\)}
\STATE \(\bm{\mathcal{X}}_{k} \in \mathbb{R} ^{n_{1} \cdots n_{i-1}m_{i+1} \cdots m_{d} \times C_{k}^{(\bm{\mathcal{B}})} \times m_{i}C_{k}^{(\bm{\mathcal{A}})}} \longleftarrow\) reshape \(\bm{\mathcal{X}}_{k}\);
\STATE \(\bm{\mathcal{X}}_{k} \in \mathbb{R} ^{n_{1} \cdots n_{i-1}m_{i+1} \cdots m_{d} \times C_{k}^{(\bm{\mathcal{B}})}} \longleftarrow \bm{\mathcal{X}}_{k} \bullet {\rm vec}(\bm{A}_{k}^{(i)})\);
\STATE \(\bm{\mathcal{X}}_{k} \in \mathbb{R} ^{n_{1} \cdots n_{i-1}m_{i+1} \cdots m_{d} \times n_{i}} \longleftarrow \bm{\mathcal{X}}_{k} \bullet \left(\bm{B}_{k}^{(i)}\right)^{\rm T}\);
\ENDIF
\ENDFOR
\IF{\(i\) mod \(2 = 0\)}
\STATE \(\bm{\mathcal{X}} \longleftarrow \sum _{k} {\bm{\mathcal{X}}_{k}}\)
\ENDIF
\ENDFOR
\RETURN \(\bm{\mathcal{Y}} \longleftarrow \bm{\mathcal{X}}\).
\end{algorithmic}
\end{algorithm}

Specifically, Algorithm \ref{Alg_KCP_2} describes the entire procedure of repeatedly using Theorem \ref{Thm_Multiply} to calculate \(\bm{\mathcal{X}} \bullet \bm{\mathcal{W}}\), which can also be seen in Figure \ref{Fig_multiplication}(b) vividly. Nevertheless, stacking \(K\) matrices \(\bm{A}_{k}^{(i)}\) or \(\bm{B}_{k}^{(i)}\) to a single one just like Equation (\ref{Eq_W_i}) is not suggested, since \(KC^{(\bm{\mathcal{A}})} \cdot KC^{(\bm{\mathcal{B}})}\) operations will be produced during the Kronecker product in Equation (\ref{Eq_RelaxFast}), contrarily, only \(KC^{(\bm{\mathcal{A}})}C^{(\bm{\mathcal{B}})}\) operations are needed in Algorithm \ref{Alg_KCP_2}.

\subsection{Computation Complexity}

It is easy to work out that the whole computation complexity of Algorithm \ref{Alg_KCP} should be
\begin{equation}\label{Eq_compcomplex}
\mathcal{O}\left(d{\rm max}\{m,n\}^{d+1}C^{(\bm{\mathcal{A}})}C^{(\bm{\mathcal{B}})}K\right).
\end{equation}
This result is more efficient than Equation (\ref{Eq_compcomplex_naive}), and is very close to the computation complexity of the multiplication in TT format, which is about \(\mathcal{O}\left(d{\rm max}\{m,n\}^{d+1}r^{2}\right)\) according to \cite{Novikov_2015_TT} where \(r\) is the maximum TT rank. The underlying reason from the other perspective is that Algorithm \ref{Alg_KCP} allows the input \(\bm{\mathcal{X}}\) to connect to all the factor matrices \(\bm{W}^{(i)}\) directly as shown in Figure \ref{Fig_mul_tng}. Thanks to the characteristic of CP, we could assume the nonzero superdiagonal line in \(\bm{\mathcal{I}}^{(\bm{\mathcal{W}})}\) is \(\bm{1}_{C^{(\bm{\mathcal{W}}})}\) \cite{Phan_2020_CP1Vector} and can be resolved into each \(\bm{W}^{(i)}\), so that every \(C^{(\bm{\mathcal{W}})}\) could be contracted one by one to relax the curse of dimensionality of CP.

\begin{figure}
\centering
\includegraphics[width=0.49\textwidth]{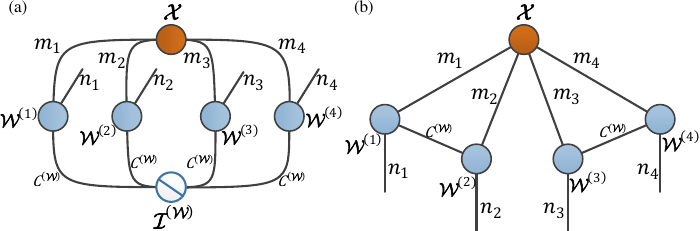}
\caption{Tensor network graphs of multiplication between input and weight in KCP format under (a) naive strategy and (b) Algorithm \ref{Alg_KCP}.}
\label{Fig_mul_tng}
\end{figure}

However, as we have already mentioned, the relaxed fast multiplication defined in Theorem \ref{Thm_Multiply} has more efficient computation complexity, since any accumulation of \(C^{(\bm{\mathcal{A}})}\) and \(C^{(\bm{\mathcal{B}})}\) is avoided. Therefore, Algorithm \ref{Alg_KCP_2} is faster than Algorithm \ref{Alg_KCP} under most circumstances at least in theory. The theorem below gives the specific computation complexity of Algorithm \ref{Alg_KCP_2} and explains that its upper bound will not exceed Equation (\ref{Eq_compcomplex}).

\begin{theorem}[Relaxed Computation Complexity]\label{Thm_CompComp}
The computation complexity of Algorithm \ref{Alg_KCP_2} should be
\begin{equation}\label{Eq_compcomplex_relax}
\mathcal{O}\left(d{\rm max}\{m,n\}^{d}\left(\frac{1}{2}C^{(\bm{\mathcal{A}})} + \frac{1}{2}C^{(\bm{\mathcal{B}})} + C^{(\bm{\mathcal{A}})}C^{(\bm{\mathcal{B}})}\right)K\right),
\end{equation}
whose upper bound is defined in Equation (\ref{Eq_compcomplex}) which is the computation complexity of Algorithm \ref{Alg_KCP}.
\end{theorem}

\begin{proof}
The proof can be found in Appendix.
\end{proof}

Undoubtedly, the computation complexity of either Equation (\ref{Eq_compcomplex}) or (\ref{Eq_compcomplex_relax}) is much lower than that in Equation (\ref{Eq_compcomplex_naive}). Still and all, because of the variant and elusive hardware/software environment, Algorithm \ref{Alg_KCP_2} does not always run faster than Algorithm \ref{Alg_KCP}, particularly in Python framework. We suggest to choose the appropriate algorithm based on the physical truth.

\begin{table}[!htbp]
\caption{Comparison of space and computation complexity among the weight matrix in Ori.(original), TT, BT, TR, HT and KCP format.}
\label{Table_complex_comparison}
\centering
\setlength{\tabcolsep}{1pt}
\begin{tabular}{l l l}
\hline
Format & Space & Computation \\
\hline
Ori. & \(\mathcal{O}\left((mn)^{d}\right)\) & \(\mathcal{O}\left((mn)^{d}\right)\) \\
TT & \(\mathcal{O}\left((d-2)mnr^{2}+2mnr\right)\) & \(\mathcal{O}\left(d{\rm max}\{m,n\}^{d+1}r^{2}\right)\) \\
BT & \(\mathcal{O}\left((dmnr+r^{d})P\right)\) & \(\mathcal{O}\left((d{\rm max}\{m,n\}^{d+1}+n^{d})r^{d}P\right)\) \\
TR & \(\mathcal{O}\left(d(m+n)r^{2}\right)\) & \(\mathcal{O}\left(d(m^{d}+n^{d})r^{3}\right)\) \\
HT & \(\mathcal{O}\left((d-1)r^{3}+dmnr\right)\) & \(\mathcal{O}\left((2d-1){\rm max}\{m,n\}^{d+1}r^{1+{\rm log}_{2}d}\right)\) \\
KCP & \(\mathcal{O}\left(d(m+n)rK\right)\) & \(\mathcal{O}\left(d{\rm max}\{m,n\}^{d}(r+r^2)K\right)\) \\
\hline
\end{tabular}
\end{table}

To highlight the possible superiority of KCP in the aspect of complexity, we make Table \ref{Table_complex_comparison} to give the complete comparison of space and computation complexity among TT, BT, TR, HT and our KCP according to Equation (\ref{Eq_tt}), (\ref{Eq_btd}), (\ref{Eq_tr}), and (\ref{Eq_ht}), respectively. For expressing concisely, we denote \(r=C^{\bm{\mathcal{A}}}=C^{\bm{\mathcal{B}}}\) for KCP. It is easy to realize that in Table \ref{Table_complex_comparison} the complexity of KCP might be lower than other kinds of tensor decomposition, since there is no \(mn\) and \(r\) with high power in the space complexity of KCP, and no combination of \({\rm max}\{m,n\}^{d+1}\) and \(r\) with high power in the computation complexity of KCP. 

\begin{figure}[!htbp]
\centering
\subfigure[]{\includegraphics[width=0.4\textwidth]{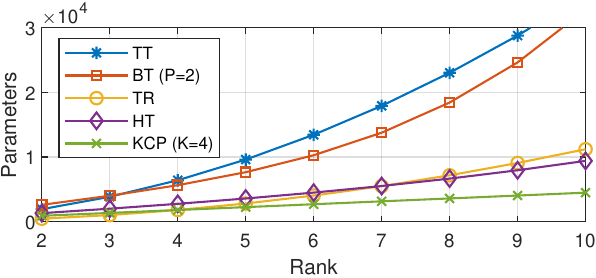}}
\subfigure[]{\includegraphics[width=0.4\textwidth]{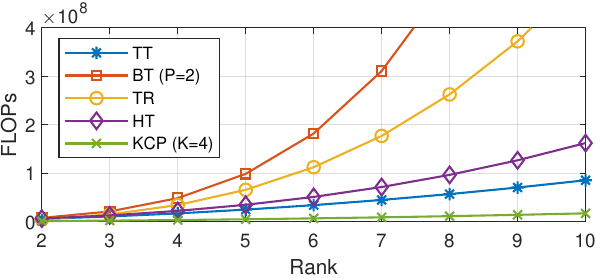}}
\caption{Variation curves of (a) parameters and (b) FLOPs in regard to ranks of TT, BT, TR, HT and KCP. For drawing easily, we let \(d=4\), \(m=20\), \(n=8\), and \(r=C^{\bm{\mathcal{A}}}=C^{\bm{\mathcal{B}}}\).}
\label{Fig_complexity}
\end{figure}

For comparing more intuitively, we also display Figure \ref{Fig_complexity} to check their trends of concrete parameters and FLOPs as the rank grows. Specifically, we assume \(P=2\) for BT according to \cite{Ye_2018_BTD} and \(K=4\) for KCP. Synthetically, both space and computation complexity of KCP are obviously more efficient than other tensor formats especially under high ranks. That is, our KCP takes making both space and computation complexity efficient in consideration.

\subsection{KCP-RNN}

Following the related works \cite{Yang_2017_TTRNN,Ye_2018_BTD,Pan_2019_TRRNN,Yin_2020_HTRNN}, in most cases only input matrices \(\bm{W}_{\theta}\) defined in Equation (\ref{Eq_lstm}) are necessary to be compressed for tensor-decomposed RNNs. Thus, making Algorithm \ref{Alg_KCP} as an example, the multiplication \(\bm{W}_{\theta}\bm{x}(t)\) could be transformed into
\begin{equation}\label{Eq_kcp_rnn}
{\rm KCP}(\bm{W}_{\theta}, \bm{x}(t)) = \bm{\mathcal{X}}(t) \bullet \bm{W}_{\theta}^{(1)} \bullet \bm{W}_{\theta}^{(2)} \bullet \cdots \bullet \bm{W}_{\theta}^{(d)}
\end{equation}
which is similar to Equation (\ref{Eq_multiplication}) but the last unnecessary \(\bm{\mathcal{I}}^{(\bm{\mathcal{W}})}\) is omitted because here we mainly consider Algorithm \ref{Alg_KCP} and \(d\) is even. Note that \({\rm KCP}(\bm{W}_{\theta}, \bm{x}(t))\) could also be computed under Algorithm \ref{Alg_KCP_2} which is cumbersome to describe in formulation. Naturally, not only LSTMs, Equation (\ref{Eq_kcp_rnn}) can also be used to compress GRUs.
% \section{RNNs in KCP Format} end
%%%%%%%%%%%%%%%%%%%%%%%%%%%%%%%%%%%%%%%%%%%%%%%%%%%%%%%%%%%%%%%%%%%%%%%%%%%%%%%%%%%%%%%%%%%

%%%%%%%%%%%%%%%%%%%%%%%%%%%%%%%%%%%%%%%%%%%%%%%%%%%%%%%%%%%%%%%%%%%%%%%%%%%%%%%%%%%%%%%%%%%
\section{Experiments}\label{sec:Exp}

\begin{table*}
\caption{Comparison of performance among tensor-decomposed LSTMs in original, TT, BT, TR, HT and KCP format based on UCF11 dataset. ``Avg Acc'' is the average accuracy of the last 100 epochs in our own experiments. ``Parameters'' records the parameters of the matrices \(\bm{W}_{\theta}\), but ``MFLOPs'' represents the operations of the whole LSTM. Note that both the best and the second best indicators are highlighted.}
\label{Table_ucf11_comparison_1}
\centering
\setlength{\tabcolsep}{10pt}
\begin{tabular}{l l l l l l l l}
\hline
Performance & Ori. & TT & BT & TR & HT & (4,4,2)-KCP & (4,2,2)-KCP \\
\hline
Top-1 Acc (Ours) & 71.9 & 87.2 & \textbf{88.1} & 86.9 & 87.8 & \textbf{88.1} & 86.9 \\
Avg Acc (Ours) & 67.4 & 85.2 & \textbf{85.7} & 84.5 & 85.1 & \textbf{86.0} & 84.3 \\
Parameters & 58,982,400 & 11,904 & 10,496 & 6,300 & 4,992 & \textbf{4,736} & \textbf{2,624} \\
Compression Ratio & - & 4,955\(\times\) & 5,620\(\times\) & 9,362\(\times\) & 11,815\(\times\) & \textbf{12,454\(\times\)} & \textbf{22,478\(\times\)} \\
MFLOPs & 335.5 & 77.2 & 264.8 & 211.7 & 125.4 & \textbf{73.1} & \textbf{37.9} \\
\hline
\end{tabular}
\end{table*}

\begin{table*}
\caption{Comparison of performance among tensor-decomposed LSTMs in original, TT, BT, TR, HT and KCP format based on Youtube Celebrities Face dataset.}
\label{Table_ycf_comparison_1}
\centering
\setlength{\tabcolsep}{10pt}
\begin{tabular}{l l l l l l l l}
\hline
Performance & Ori. & TT & BT & TR & HT & (4,4,2)-KCP & (4,2,2)-KCP \\
\hline
Top-1 Acc (Ours) & 75.4 & 86.4 & 84.3 & \textbf{86.6} & \textbf{86.6} & 85.3 & 85.3 \\
Avg Acc (Ours) & 70.8 & 85.0 & 83.0 & \textbf{85.5} & \textbf{85.9} & 83.8 & 83.9 \\
Parameters & 58,982,400 & 12,800 & 12,288 & 6,300 & 5,888 & \textbf{5,632} & \textbf{3,072} \\
Compression Ratio & - & 4,608\(\times\) & 4,800\(\times\) & 9,362\(\times\) & 10,017\(\times\) & \textbf{10,473\(\times\)} & \textbf{19,200\(\times\)} \\
MFLOPs & 355.5 & 130.8 & 480.7 & 211.7 & 227.2 & \textbf{122.4} & \textbf{63.1} \\
\hline
\end{tabular}
\end{table*}

\begin{table*}
\caption{Comparison of performance among tensor-decomposed LSTMs in original, TT, BT, TR, HT and KCP format based on UCF50 dataset.}
\label{Table_ucf50_comparison_1}
\centering
\setlength{\tabcolsep}{10pt}
\begin{tabular}{l l l l l l l l l}
\hline
Performance & Ori. & TT & BT & TR & HT & (6,4,4)-KCP & (6,4,2)-KCP & (6,2,2)-KCP \\
\hline
Top-1 Acc (Ours) & 66.9 & 86.6 & \textbf{87.1} & 86.2 & 84.4 & \textbf{87.1} & \textbf{87.1} & 85.9 \\
Avg Acc (Ours) & 65.4 & 85.0 & 85.2 & 84.5 & 82.2 & \textbf{85.4} & \textbf{85.3} & 84.1 \\
Parameters & 530,841,600 & 33,408 & 31,104 & 10,032 & 12,960 & 8,640 & \textbf{7,296} & \textbf{4,320} \\
Compression Ratio & - & 15,890\(\times\) & 17,067\(\times\) & 52,915\(\times\) & 40,960\(\times\) & 61,440\(\times\) & \textbf{72,758\(\times\)} & \textbf{122,880\(\times\)} \\
MFLOPs & 3312.6 & 417.7 & 3,588.6 & 419.1 & 609.4 & 336.8 & \textbf{252.0} & \textbf{191.8} \\
\hline
\end{tabular}
\end{table*}

This section verifies the efficiency of our KCP-RNNs on UCF11, Youtube Celebrities Face, and UCF50 datasets, which have already become the benchmarks to test tensor-decomposed RNNs. All the experiments here are executed in TensorFlow. The KCP-RNNs follow the proposed Algorithm 1 since Algorithm 2 has poor running time on Python platform.

\subsection{Experiments on UCF11}

All of the related tensor-decomposed works have tested their methods to compress LSTMs on UCF11 dataset \cite{Yang_2017_TTRNN,Ye_2018_BTD,Pan_2019_TRRNN,Yin_2020_HTRNN,Wu_2020_Hybrid}, so we mainly follow them to experiment our KCP-LSTM based on this dataset. Besides, most of the training settings are also the same with these existing works.

In detail, we resize each sampled video frame from 320\(\times\)240\(\times\)3 to 160\(\times\)120\(\times\)3=57600 as a vector, and in total 6 frames are sampled as an input data. Naturally, there are 6 units along the temporal direction in the KCP-LSTM, whose single hidden layer has 256 neurons. The tensorizing shapes of the input and output are respectively set as 8\(\times\)20\(\times\)20\(\times\)18 and 4\(\times\)4\(\times\)4\(\times\)4, which follow \cite{Yang_2017_TTRNN,Ye_2018_BTD}. Correspondingly, KCP ranks are set into two groups, one contains \(K=4\), \(C^{(\bm{\mathcal{A}})}=4\), \(C^{(\bm{\mathcal{B}})}=2\) (termed as (4,4,2)-KCP), the other is lighter with \(K=4\), \(C^{(\bm{\mathcal{A}})}=2\), \(C^{(\bm{\mathcal{B}})}=2\) (termed as (4,2,2)-KCP). We also select the Adam optimizer and set the initial learning rate as 0.001 \cite{Yang_2017_TTRNN,Ye_2018_BTD,Pan_2019_TRRNN,Yin_2020_HTRNN}. The number of training epochs is 400, which is utilized by \cite{Ye_2018_BTD,Pan_2019_TRRNN} and better than traditional 100.

\begin{figure}[!htbp]
\centering
\includegraphics[width=0.45\textwidth]{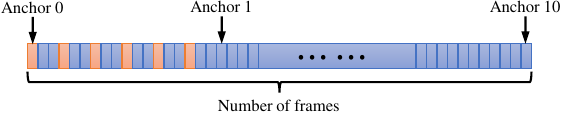}
\caption{Uniformly segmented clips for sampling validation frames. The validation video is evenly divided into 10 clips, then from each clip, 6 frames are uniformly and discretely sampled.}
\label{Fig_val_clips}
\end{figure}

We also reproduce these former works in TT, BT, TR, and HT formats to make a comprehensive comparison. Except TR whose tensorizing and ranks are specific \cite{Zhao_2018_TR,Pan_2019_TRRNN}, all these experiments utilize the shapes 8\(\times\)20\(\times\)20\(\times\)18 and 4\(\times\)4\(\times\)4\(\times\)4 with rank 4 mentioned above. Note that the experiments in \cite{Yin_2020_HTRNN} are made on the shape 8\(\times\)10\(\times\)10\(\times\)9\(\times\)8 that is not considered here for comparing conveniently and convincingly. We follow the related works to execute the 5-fold cross validation \cite{Yang_2017_TTRNN,Ye_2018_BTD,Pan_2019_TRRNN,Yin_2020_HTRNN}, where each validation sample is tested with 10 uniformly segmented clips \cite{Varol_2018_LongTerm3DCNN}, the strategy of which is explained in Figure \ref{Fig_val_clips} iconically.

The synthetic results are listed in Table \ref{Table_ucf11_comparison_1}, in which it can be abstractly learned that there is no tangible evidence to reflect different tensor formats can cause significant differences in accuracy. Besides, TR, HT, and KCP get relatively much fewer parameters, while TT and KCP achieve faster calculation.

\subsection{Experiments on Youtube Celebrities Face}

Youtube Celebrities Face is also a commonly used dataset for tensor-decomposed RNNs \cite{Yang_2017_TTRNN,Yin_2020_HTRNN}, so we try this to test our KCP-LSTM and other kinds of tensor-decomposed LSTMs as well. Most of the training settings including (4,4,2)- and (4,2,2)-KCP are also the same with those of UCF11 but a little differences, which include that, the input shape is tensorized as 4\(\times\)20\(\times\)20\(\times\)36 following \cite{Yang_2017_TTRNN}, and weight decay with 0.01 is applied for all the tensor-decomposed LSTMs.

Table \ref{Table_ycf_comparison_1} exhibits the results of these experiments, and it seems that the variant tensorizing shape of the input, i.e., from 8\(\times\)20\(\times\)20\(\times\)18 to 4\(\times\)20\(\times\)20\(\times\)36, is harmful to the tensor-decomposed LSTMs in BT and KCP formats. However, (4,2,2)-KCP can get significantly lower space and computation complexity than any other one.

\subsection{Experiments on UCF50}

UCF50 is a more complex dataset which could be regarded as the extension of UCF11, so we follow \cite{Wu_2020_Hybrid} to verify whether the similar experimental results can also be obtained on this dataset. We adopt a wider hidden layer, which has 2304 neurons, and the relatively balanced tensorizing shapes, i.e, 15\(\times\)16\(\times\)16\(\times\)15 and 8\(\times\)6\(\times\)6\(\times\)8 \cite{Wu_2020_Hybrid}. Specifically, for the reproduced LSTMs in TT, BT, and HT formats, most ranks are improved from 4 to 6 \cite{Yang_2017_TTRNN,Ye_2018_BTD,Yin_2020_HTRNN}, while for TR, most ranks are enlarged from 5 to 6 \cite{Pan_2019_TRRNN}. With regard to KCP, in total 3 different formats with \(K=6\), i.e., (6,4,4)-, (6,4,2)- and (6,2,2)-KCP, are considered. Besides, since we realize that the wider hidden layer makes faster convergence, the random left-right flipping for training data augmentation is added here, and the upper number of epochs is limited at 200. All the other training settings are the same with those on Youtube Celebrities Face dataset.

According to the results recorded in Table \ref{Table_ucf50_comparison_1}, our KCP-LSTMs can still keep the competitive accuracy compared with other kinds of tensor-decomposed networks. Meanwhile, each specific KCP-LSTM has lower parameters and MFLOPs than any other kind of tensor-decomposed one. 
%\section{Experiments} end
%%%%%%%%%%%%%%%%%%%%%%%%%%%%%%%%%%%%%%%%%%%%%%%%%%%%%%%%%%%%%%%%%%%%%%%%%%%%%%%%%%%%%%%%%%%

%%%%%%%%%%%%%%%%%%%%%%%%%%%%%%%%%%%%%%%%%%%%%%%%%%%%%%%%%%%%%%%%%%%%%%%%%%%%%%%%%%%%%%%%%%%
\section{Discussions}\label{sec:Disc}
In this section, we discuss some characteristics of KCP format, particularly the weights sharing in LSTMs and the potential for parallel computing. The comparison with other correlative state-of-the-art practices will also be given.

\subsection{Weights Sharing}

\begin{table*}
\caption{Comparison of performance among tensor-decomposed LSTMs in original, TT, BT, TR, HT and KCP format with weights sharing based on UCF11 dataset.}
\label{Table_ucf11_comparison_2}
\centering
\setlength{\tabcolsep}{10pt}
\begin{tabular}{l l l l l l l l}
\hline
Performance & Ori. & TT & BT & TR & HT & (4,4,2)-KCP & (4,2,2)-KCP \\
\hline
Top-1 Acc (Ours) & 71.9 & \textbf{88.4} & 85.3 & 86.6 & 86.0 & \textbf{88.1} & 85.0 \\
Avg Acc (Ours) & 67.4 & \textbf{86.2} & 83.1 & 84.4 & 83.7 & \textbf{85.2} & 82.4 \\
Parameters & 58,982,400 & 3,360 & 3,392 & 1,725 & \textbf{1,632} & 1,664 & \textbf{994}\\
Compression Ratio & - & 17,554\(\times\) & 17,389\(\times\) & 34,193\(\times\) & \textbf{36,141\(\times\)} & 35,446\(\times\) & \textbf{59,338\(\times\)} \\
MFLOPs & 335.5 & \textbf{37.1} & 100.6 & 118.9 & 49.1 & 52.6 & \textbf{27.2} \\
\hline
\end{tabular}
\end{table*}

\begin{table*}
\caption{Comparison of performance among tensor-decomposed LSTMs in original, TT, BT, TR, HT and KCP format with weights sharing based on Youtube Celebrities Face dataset.}
\label{Table_ycf_comparison_2}
\centering
\setlength{\tabcolsep}{10pt}
\begin{tabular}{l l l l l l l l}
\hline
Performance & Ori. & TT & BT & TR & HT & (4,4,2)-KCP & (4,2,2)-KCP \\
\hline
Top-1 Acc (Ours) & 75.4 & 86.1 & 84.8 & \textbf{87.2} & 86.4 & \textbf{86.6} & 86.4 \\
Avg Acc (Ours) & 70.8 & 85.0 & 83.4 & \textbf{85.6} & 85.0 & \textbf{85.3} & 84.4\\
Parameters & 58,982,400 & 3,392 & 3,456 & 1,725 & \textbf{1,664} & 1,696 & \textbf{960} \\
Compression Ratio & - & 17,389\(\times\) & 17,067\(\times\) & 34,193\(\times\) & \textbf{35,446\(\times\)} & 34,777\(\times\) & \textbf{61,440\(\times\)} \\
MFLOPs & 355.5 & \textbf{50.5} & 154.5 & 118.9 & 74.6 & 81.6 & \textbf{41.9} \\
\hline
\end{tabular}
\end{table*}

\begin{table*}
\caption{Comparison of performance among tensor-decomposed LSTMs in original, TT, BT, TR, HT and KCP format with weights sharing based on UCF50 dataset.}
\label{Table_ucf50_comparison_2}
\centering
\setlength{\tabcolsep}{10pt}
\begin{tabular}{l l l l l l l l l}
\hline
Performance & Ori. & TT & BT & TR & HT & (6,4,4)-KCP & (6,4,2)-KCP & (6,2,2)-KCP \\
\hline
Top-1 Acc (Ours) & 66.9 & 86.7 & 86.3 & 86.4 & \textbf{87.5} & \textbf{87.0} & 86.6 & 86.7 \\
Avg Acc (Ours) & 65.4 & 85.0 & 84.2 & 84.5 & \textbf{85.8} & 85.1 & 85.0 & \textbf{85.5} \\
Parameters & 530,841,600 & 10,512 & 12,096 & \textbf{2,724} & 5,400 & 3,816 & 3,192 & \textbf{1,908} \\
Compression Ratio & - & 50,499\(\times\) & 43,886\(\times\) & \textbf{194,876\(\times\)} & 98,304\(\times\) & 139,109\(\times\) & 166,304\(\times\) & \textbf{278,219\(\times\)} \\
MFLOPs & 3312.6 & 250.8 & 1,092.3 & 293.7 & 297.8 & 257.8 & \textbf{210.1} & \textbf{169.3} \\
\hline
\end{tabular}
\end{table*}

Yang et al. \cite{Yang_2017_TTRNN} first realized that the weights sharing could be taken into account during compressing TT-LSTM, i.e., concatenate all the input matrices \(\bm{W}_{f}\), \(\bm{W}_{i}\), \(\bm{W}_{z}\), and \(\bm{W}_{o}\) defined in Equation (\ref{Eq_lstm}) to a larger matrix \(\bm{W} \in \mathbb{R} ^{4M \times N}\) to compress as shown in Figure \ref{Fig_weight_share_sketch}(a). Likewise, for KCP-LSTM, by consulting Equation (\ref{Eq_kt_to_kcp}) strictly, only \(\bm{A}_{k}^{(1)} \in \mathbb{R} ^{m_{1} \times C_{k}^{(\bm{\mathcal{A}})}}\) is not shared since \(\bm{A}_{k}^{(1)}\) and \(\bm{B}_{k}^{(1)}\) are two separate matrices so that we can just concatenate \(\bm{A}_{k}^{(1)}\) generally. However, in practice we discover that just making \(\bm{A}_{k}^{(1)}\) independent is not enough, and considering both \(\bm{A}_{k}^{(1)}\) and \(\bm{B}_{k}^{(1)}\), i.e., \(\bm{W}^{(1)}\), to be unshared can make up the accuracy loss significantly as shown in Figure \ref{Fig_weights_share}. The corresponding schematic diagram of our modified weights sharing of KCP-LSTM is illustrated in Figure \ref{Fig_weight_share_sketch}(b), in which we have not drawn all \(\bm{A}_{k}^{(i)} \otimes \bm{B}_{k}^{(i)}\) on account of the limitation of writing space.

\begin{figure}[!htbp]
\centering
\includegraphics[width=0.49\textwidth]{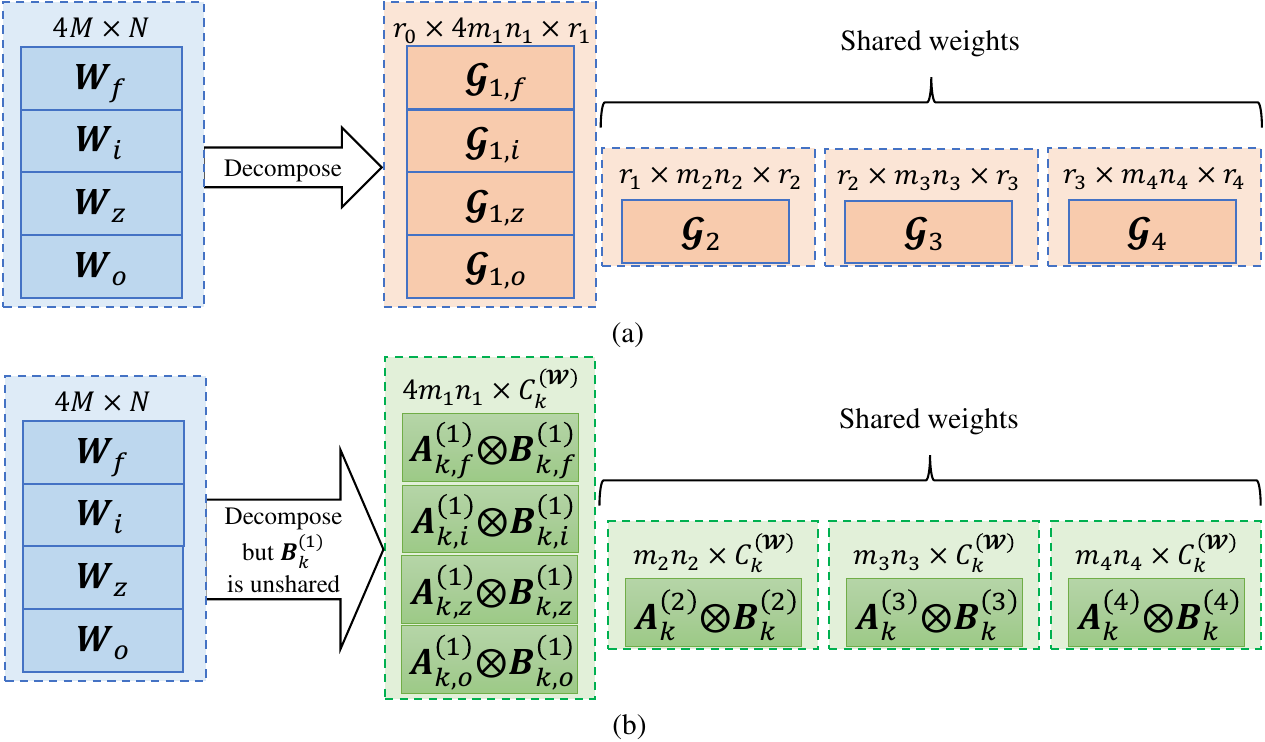}
\caption{Schematic diagrams of weights sharing of (a) TT-LSTM and (b) KCP-LSTM. Note that there are in total \(K\) components for each factor matrix \(\bm{W}^{(i)}\), but only the \(k\)th component is drawn here to delegate all.}
\label{Fig_weight_share_sketch}
\end{figure}

\begin{figure}[!htbp]
\centering
\subfigure[]{\includegraphics[width=0.4\textwidth]{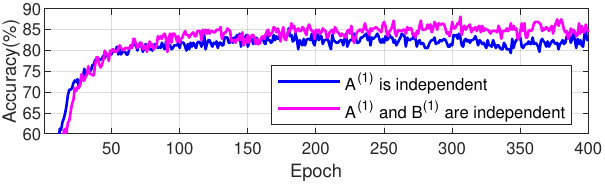}}
\subfigure[]{\includegraphics[width=0.4\textwidth]{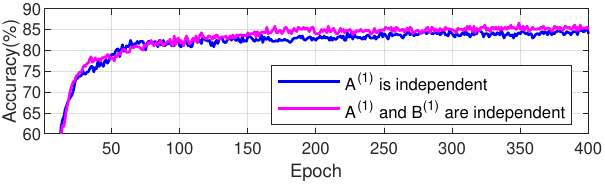}}
\caption{Learning curves of (4,4,2)-KCP-LSTMs under different weights sharing on (a) UCF11 and (b) Youtube Celebrities Face datasets.}
\label{Fig_weights_share}
\end{figure}

The results of tensor-decomposed LSTMs with weights sharing on UCF11, Youtube Celebrities Face, and UCF50 datasets are displayed in Table \ref{Table_ucf11_comparison_2}, \ref{Table_ycf_comparison_2}, and \ref{Table_ucf50_comparison_2}, respectively. The compression ratio of each tensor-decomposed LSTM is evidently improved compared with Table \ref{Table_ucf11_comparison_1}, \ref{Table_ycf_comparison_1} and \ref{Table_ucf50_comparison_1}. It is amazing that the compression ratio of (6,2,2)-KCP-LSTM, which could still maintain the accuracy, is about \(2.8\times 10^5\) times, and is extremely higher than any other kind of tensor decomposition.

\subsection{Compression Effect}

\begin{figure*}
\centering
\subfigure[]{\includegraphics[width=0.3\textwidth]{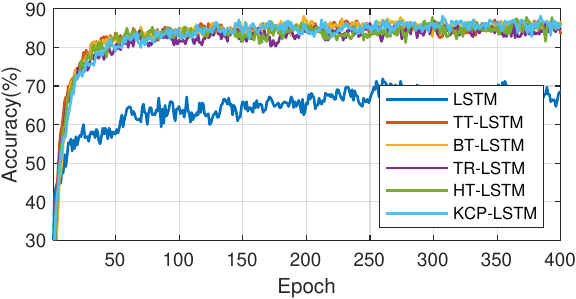}}
\subfigure[]{\includegraphics[width=0.3\textwidth]{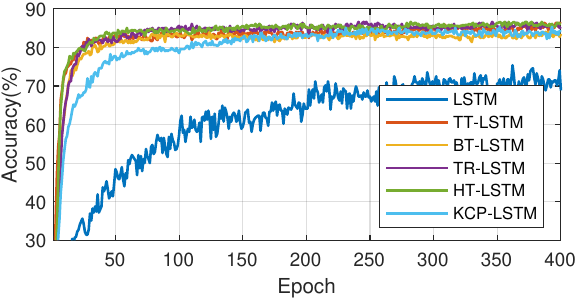}}
\subfigure[]{\includegraphics[width=0.3\textwidth]{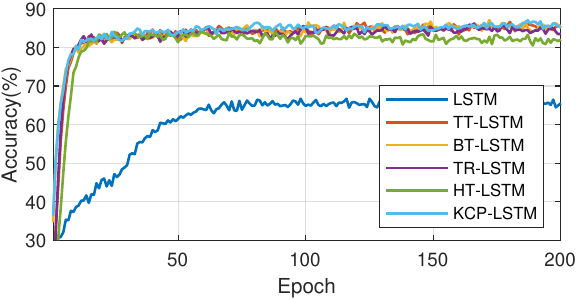}}
\subfigure[]{\includegraphics[width=0.3\textwidth]{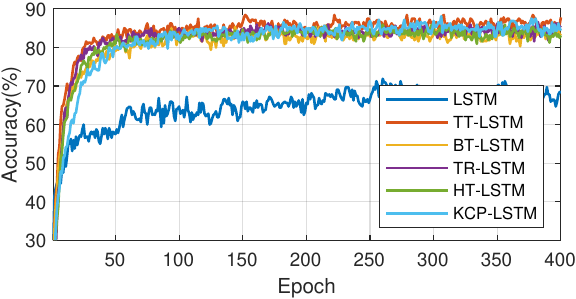}}
\subfigure[]{\includegraphics[width=0.3\textwidth]{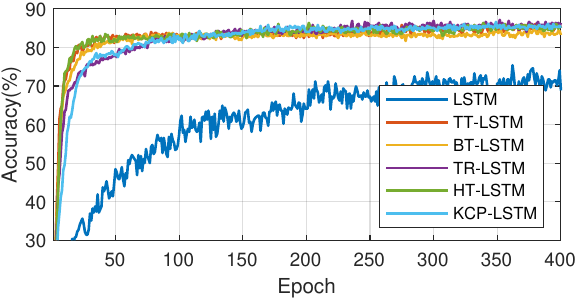}}
\subfigure[]{\includegraphics[width=0.3\textwidth]{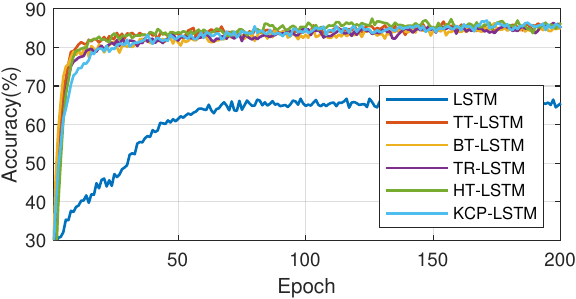}}
\caption{Learning curves of tensor-decomposed LSTMs on (a) UCF11, (b) Youtube Celebrities Face, and (c) UCF50 without weights sharing; on (d) UCF11, (e) Youtube Celebrities Face, and (f) UCF50 with weights sharing. Here we only show the (4,4,2)- and (6,4,4)-KCP-LSTM.}
\label{Fig_curves}
\end{figure*}

All of our experiments represent the compression effect of multiple tensor-decomposed LSTMs comprehensively. We find that it is arbitrary to claim which kind of tensor decomposition method is the best on preserving information by scanning Table \ref{Table_ucf11_comparison_1}, \ref{Table_ycf_comparison_1}, \ref{Table_ucf50_comparison_1}, \ref{Table_ucf11_comparison_2}, \ref{Table_ycf_comparison_2} and \ref{Table_ucf50_comparison_2}. Specifically, as we mentioned in Section \ref{sec:Intro}, reviewing related works \cite{Yang_2017_TTRNN,Ye_2018_BTD,Pan_2019_TRRNN,Yin_2020_HTRNN} might draw a conclusion that the level of superiority is TT \(<\) BT \(<\) TR \(<\) HT, which is perhaps misleading according to our own practices, where TT-LSTM with weights sharing even gets the best accuracy on UCF11. The experimental results also reflect that the lack of uniqueness of KT \cite{Phan_2012_KTD1} has no significant influence on the compression effect of KCP-RNNs. Synthetically, KCP-LSTM has its particular advantages in both space and computation complexity at least for the standard applications without elaborate weights sharing.

Learning curves illustrated in Figure \ref{Fig_curves} further verify our opinion intuitively since the curves of different tensor-decomposed formats are all very close. However, it seems that KCP-LSTM converges a bit slower sometimes, but its final level of accuracy has no significant difference compared with other tensor-decomposed LSTMs.

\begin{figure}
\centering
\subfigure[]{\includegraphics[width=0.22\textwidth]{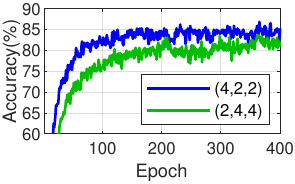}}
\subfigure[]{\includegraphics[width=0.22\textwidth]{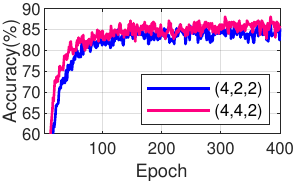}}
\subfigure[]{\includegraphics[width=0.22\textwidth]{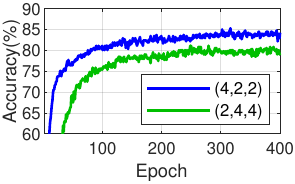}}
\subfigure[]{\includegraphics[width=0.22\textwidth]{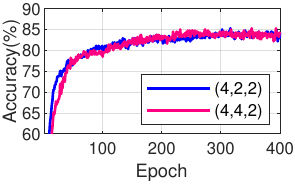}}
\subfigure[]{\includegraphics[width=0.22\textwidth]{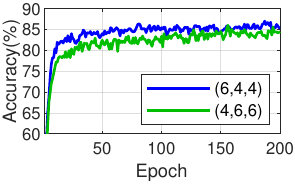}}
\subfigure[]{\includegraphics[width=0.22\textwidth]{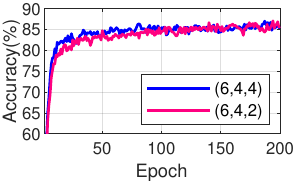}}
\caption{Learning curves of KCP-LSTMs with different KT rank but the same amounts of parameters, or the variant parameters under the same KT rank, on (a)(b) UCF11, (c)(d) Youtube Celebrities Face, and (e)(f) UCF50 datasets.}
\label{Fig_diff_shape}
\end{figure}

\subsection{KT Rank}

Combining with the inspiration from Lemma \ref{lem_rankk}, we discover that KT rank \(K\) is more important than CP ranks \(C^{\bm{\mathcal{A}}}\) and \(C^{\bm{\mathcal{B}}}\). For example, it is easy to work out that the amount of parameters of (4,2,2)-KCP and (2,4,4)-KCP are the same, however, the former runs much better than the latter one as demonstrated in Figure \ref{Fig_diff_shape}(a) and (c). For relatively complex UCF50 dataset, the same phenomenon can be seen in Figure \ref{Fig_diff_shape}(e) by comparing (6,4,4)-KCP and (4,6,6)-KCP. This means that \(K\) makes more contribution to the expressive ability of KCP-LSTMs, and in comparison the number of parameters is the secondary, since the learning curves grouped in Figure \ref{Fig_diff_shape}(b), (d), and (f) are much closer separately.

More importantly, with the same number of parameters, higher \(K\) brings more efficient computation complexity. According to Table \ref{Table_ucf11_comparison_1} and \ref{Table_ycf_comparison_1}, the MFLOPs of (4,2,2)-KCP-LSTMs on UCF11 and Youtube Face Celebrities are 37.9 and 63.1 respectively, while those of (2,2,4)-KCP-LSTMs could be separately calculated as 60.9 and 109.1. Likewise, (6,4,4)-KCP-LSTM on UCF50 dataset has 336.8 MFLOPs recorded in Table \ref{Table_ucf50_comparison_1}, whereas the MFLOPs of (4,6,6)-KCP-LSTM are about 417.5.

\subsection{Potential Parallel Computing}

Since both BT and KCP have the form of sum, these two tensor decomposition methods become the most probable formats that could be dealt with concurrently. Lemma \ref{lem_KTRank} hints that KCP has better potential for parallel computing than BT. Space complexity listed in Table \ref{Table_complex_comparison} could give a more obvious explanation. Concretely, even if we do not consider the kernel tensor in BT, i.e., \(r^{d}\) is removed, for a fair comparison, there should be
\begin{equation}
\mathcal{O}\left(dmnrP\right) = \mathcal{O}\left(d(m+n)CK\right)
\end{equation}
where \(C={\rm max}\{C^{(\bm{\mathcal{A}})},C^{(\bm{\mathcal{B}})}\}\). Further, in practice, Tucker rank \(r\) is set as 4 \cite{Ye_2018_BTD} while CP rank \(C\) is also 4 in most of our experiments. Therefore, we have
\begin{equation}
K = \frac{mn}{m+n}P
\end{equation}
in which \((mn)/(m+n) \geq 1\) resulting \(K \geq P\), is usually established since \({\rm min}\{m,n\} \geq 2\) in tensorizing. For example, BT-LSTM in \cite{Ye_2018_BTD} has \(P=2\) on UCF11 dataset, while our KCP-LSTM has \(K=4\), and even if so KCP-LSTM still gets lower parameters according to Table \ref{Table_ucf11_comparison_1} and \ref{Table_ucf11_comparison_2}.

What needs to be emphasized is that, standard KCP format defined in Equation (\ref{Eq_kcp}) should not be considered for parallel computing. Contrarily, directly using Equation (\ref{Eq_general_ktd}) by substituting \(\bm{\mathcal{A}}_{k}\) and \(\bm{\mathcal{B}}_{k}\) with Equation (\ref{Eq_kt_to_kcp}) is the only choice. Nevertheless, the strategy of Algorithm \ref{Alg_KCP} could still be transplanted into the procedure of the multiplication between \(\bm{\mathcal{X}}\) and each \(\bm{W}_{k}^{(i)} = \bm{A}_{k}^{(i)} \otimes \bm{B}_{k}^{(i)}\) to gradually remove the CP ranks \(C_{k}^{(\bm{\mathcal{W}})}=C_{k}^{(\bm{\mathcal{A}})}C_{k}^{(\bm{\mathcal{B}})}\). If so, the ideal computation complexity will be further reduced from \(\mathcal{O}\left(d{\rm max}\{m,n\}^{d+1}C^{(\bm{\mathcal{A}})}C^{(\bm{\mathcal{B}})}K\right)\) to \(\mathcal{O}\left(d{\rm max}\{m,n\}^{d+1}C^{(\bm{\mathcal{A}})}C^{(\bm{\mathcal{B}})}\right)\) according to Equation (\ref{Eq_compcomplex}). The intuitionistic description of the parallel computing in KCP format is exhibited in Figure \ref{Fig_multiplication_parallel}, which is very similar to Figure \ref{Fig_multiplication}(b), and both of them should carry out \(\sum_{k} \bm{\mathcal{X}}_{k}\) when the even number of \(i\) is passed because of the limitation of Theorem 2. Spontaneously, Algorithm \ref{Alg_KCP_2} is also a natural choice for parallel computing.

\begin{figure}[!htbp]
\centering
\includegraphics[width=0.42\textwidth]{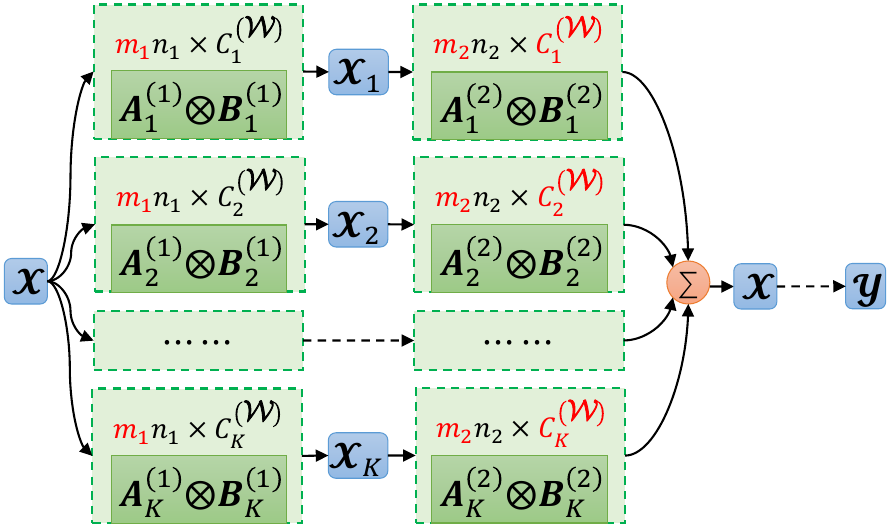}
\caption{Parallel multiplication between input and weight in KCP format on the basis of Algorithm \ref{Alg_KCP}.}
\label{Fig_multiplication_parallel}
\end{figure}

Here we use KCP-weights in our experiments on UCF11 and UCF50 datasets, i.e., \((8 \cdot 4) \times (20 \cdot 4) \times (20 \cdot 4) \times (18 \cdot 4)\) with \(K=4\) and \((15 \cdot 8) \times (16 \cdot 6) \times (16 \cdot 6) \times (15 \cdot 8)\) with \(K=6\), to design a simple simulation to test the power of parallel computing of KCP. We let the CP rank to be \(C=C^{(\bm{\mathcal{A}})}=C^{(\bm{\mathcal{B}})}\) and set it from 2 to 100 to test their inference time of multiplications under serial and parallel strategies respectively, of which the former is referred to Figure \ref{Fig_multiplication}(a) and the latter is consulted with Figure \ref{Fig_multiplication_parallel} by utilizing multi-process. As demonstrated in Figure \ref{Fig_parallel_test}, it is clear that the parallel multiplication has obvious superiority over the normal serial algorithm, and this advantage is enlarging as the scale of rank grows.

\begin{figure}[!htbp]
\centering
\includegraphics[width=0.35\textwidth]{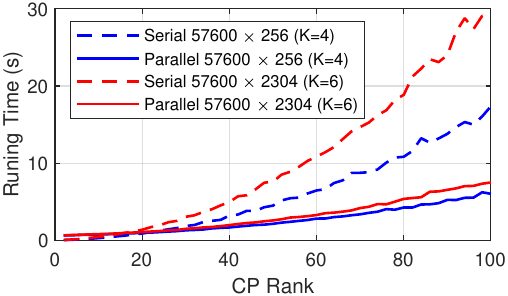}
\caption{Serial vs parallel multiplication between input and weight in KCP format on the basis of Algorithm \ref{Alg_KCP}.}
\label{Fig_parallel_test}
\end{figure}

However, when the CP rank is relatively low, the superiority of parallel computing is missing. The reason is that the cost of building sub-processes might be heavier than the saved MFLOPs by parallel computing. To implement the full power of parallel computing of KCP, we deem that combining with the specific hardware architecture to minimize the cost of managing multi-process or multi-thread is necessary in future works.

\subsection{Comparison with Typical Works}

Finally, here we give a brief comparison between our results and other typical works in Table \ref{Table_comparison_sota} to highlight the superiority of KCP-RNNs. Since most of our training settings follow the previous typical works \cite{Yang_2017_TTRNN,Ye_2018_BTD,Pan_2019_TRRNN,Yin_2020_HTRNN,Wu_2020_Hybrid} that all utilized weights sharing, our results listed here are just copied from Table \ref{Table_ucf11_comparison_2}, \ref{Table_ycf_comparison_2} and \ref{Table_ucf50_comparison_2}. For fairness, the amount of parameters of the entire tensor-decomposed LSTMs is shown rather than only \(\bm{W}_{\theta}\). Obviously, in most cases, performances of our KCP-LSTMs could be matched with those typical works, and further the RNNs in KCP format can give consideration to both efficient space and computation complexity. Note that we mainly focus on the practices using end-to-end training in view of our own experiments, but unfortunately, most of them on UCF50 need pre-training.

\begin{table}[!htbp]
\caption{Comparing our experimental results with partial typical works on UCF11, Youtube Celebrities Face, and UCF50 datasets.}
\label{Table_comparison_sota}
\centering
\setlength{\tabcolsep}{5pt}
\begin{tabular}{l l l l}
\hline
\multicolumn{4}{c}{UCF11} \\
\hline
Method & Top-1 Acc & Parameters & MFLOPs \\
\hline
Visual Attention (2016) \cite{Sharma_2016_UCF11} & 85.0 & 108M & - \\
TT-LSTM (2017) \cite{Yang_2017_TTRNN} & 79.6 & 0.268M & 37.1 \\
TT-GRU (2017) \cite{Yang_2017_TTRNN} & 81.3 & 0.203M & 31.2 \\
BT-LSTM (2018) \cite{Ye_2018_BTD} & 85.3 & 0.268M & 100.6 \\
TR-LSTM (2019) \cite{Pan_2019_TRRNN} & 86.9 & 0.267M & 118.9 \\
HT-LSTM (2020) \cite{Wu_2020_Hybrid} & 79.7 & 7.4M & - \\
5th-order HT-LSTM (2020) \cite{Yin_2020_HTRNN} & 87.2 & 0.266M & - \\
Ours (4,4,2)-KCP-LSTM & \textbf{88.1} & \textbf{0.267M} & \textbf{52.6} \\
Ours (4,2,2)-KCP-LSTM & \textbf{85.0} & \textbf{0.266M} & \textbf{27.2} \\
\hline
\multicolumn{4}{c}{Youtube Celebrities Face} \\
\hline
Method & Top-1 Acc & MPs & MFLOPs \\
\hline
Multi-manifold (2015) \cite{Lu_2015_MultiManifold} & 78.5 & 0.1M & - \\
RRNN (2016) \cite{Li_2016_RRNN} & 86.6 & 38M & 381 \\
TT-LSTM (2017) \cite{Yang_2017_TTRNN} & 75.5 & 0.278M & 50.5 \\
TT-GRU (2017) \cite{Yang_2017_TTRNN} & 80.0 & 0.212M & 44.5 \\
5th-order HT-LSTM (2020) \cite{Yin_2020_HTRNN} & 88.1 & 0.275M & - \\
Ours (4,4,2)-KCP-LSTM & \textbf{86.6} & \textbf{0.267M} & \textbf{81.6} \\
Ours (4,2,2)-KCP-LSTM & \textbf{86.4} & \textbf{0.266M} & \textbf{41.9} \\
\hline
\multicolumn{4}{c}{UCF50} \\
\hline
Method & Top-1 Acc & MPs & MFLOPs \\
\hline
VGG-F (2017) \cite{Banerjee_2017_VGGFPool} & 78.0 & 74M & 724 \\
SDMA-TDD (2019) \cite{Xu_2019_SDMA_TDD} & 89.0 & 59M & - \\
Fusion Feature (2019) \cite{Wang_2019_Fusion} & 91.1 & 111M & - \\
HT-LSTM (2020) \cite{Wu_2020_Hybrid} & 76.6 & 7.4M & - \\
Ours (6,4,4)-KCP-LSTM & \textbf{87.0} & \textbf{5.4M} & \textbf{257.8} \\
Ours (6,4,2)-KCP-LSTM & \textbf{86.6} & \textbf{5.4M} & \textbf{210.1} \\
Ours (6,2,2)-KCP-LSTM & \textbf{86.7} & \textbf{5.4M} & \textbf{169.3} \\
\hline
\end{tabular}
\end{table}
% \section{Discussions} end
%%%%%%%%%%%%%%%%%%%%%%%%%%%%%%%%%%%%%%%%%%%%%%%%%%%%%%%%%%%%%%%%%%%%%%%%%%%%%%%%%%%%%%%%%%%

%%%%%%%%%%%%%%%%%%%%%%%%%%%%%%%%%%%%%%%%%%%%%%%%%%%%%%%%%%%%%%%%%%%%%%%%%%%%%%%%%%%%%%%%%%%
\section{Conclusions}\label{sec:Conc}

This paper introduces a novel tensor-decomposed KCP-RNN, which is based on the knowledge of KT and KCP tensor decomposition. We also prove some important lemmas and theorems to consolidate corresponding theoretical foundations of KT and KCP. Our experiments on UCF11, Youtube Celebrities Face, and UCF50 datasets reveal that different kinds of tensor-decomposed RNNs do not have significant differences on accuracy no matter with or without weights sharing. This conclusion might correct the possible misleading opinion, i.e., each latter one is better than the former one among TT, BT, TR and HT, caused by pre-existing related works \cite{Ye_2018_BTD,Pan_2019_TRRNN,Yin_2020_HTRNN}. Nevertheless, our KCP-RNN has its own superiority that the multiplication in KCP format is good for not only lightweight space complexity, but also efficient computation complexity. Moreover, we discuss that KCP-RNN has the best potential for parallel computing, which is a promising point in future works.
% \section{Conclusions} end
%%%%%%%%%%%%%%%%%%%%%%%%%%%%%%%%%%%%%%%%%%%%%%%%%%%%%%%%%%%%%%%%%%%%%%%%%%%%%%%%%%%%%%%%%%%

% Appendix
\section*{Appendix}

\subsection{Proof of the Theorem \ref{Thm_KCP}}\label{appdx:a}
\begin{proof}
Above all, it is necessary to simplify the object of proof. Therefore, suppose this theorem is correct, according to the definition of CP, the contraction format defined in Equation (\ref{Eq_kcp}) can also be represented as the sum of rank-1 tensors like
\begin{equation}\label{Eq_thm_2_1}
\bm{\mathcal{W}} = \sum_{c=1}^{C^{(\bm{\mathcal{W}}})} {\bm{w}_{c}^{(1)} \circ \bm{w}_{c}^{(2)} \circ \cdots \circ \bm{w}_{c}^{(d)}}
\end{equation}
where \(\bm{w}_{c}^{(i)} \in \mathbb{R} ^{m_{i}n_{i}}\) comes from the \(c\)th column vector of \(\bm{W}^{(i)}\). If we stack all these \(\bm{W}^{(i)}\) as a matrix \(\overline{\bm{W}}\) like
\begin{equation}\label{Eq_thm_2_2}
\overline{\bm{W}} = \begin{bmatrix} \bm{W}^{(1)} \\ \bm{W}^{(2)} \\ \vdots \\ \bm{W}^{(d)} \end{bmatrix},
\end{equation}
then the \(c\)th rank-1 tensor \(\bm{w}_{c}^{(1)} \circ \bm{w}_{c}^{(2)} \circ \cdots \circ \bm{w}_{c}^{(d)}\) from Equation (\ref{Eq_thm_2_1}) is composed by the entries of the \(c\)th column of \(\overline{\bm{W}}\). Likewise, the first \(C_{1}^{(\bm{\mathcal{A}})}C_{1}^{(\bm{\mathcal{B}})}\) rank-1 tensors on the right side of Equation (\ref{Eq_thm_2_1}) are made up of entries of the first \(C_{1}^{(\bm{\mathcal{A}})}C_{1}^{(\bm{\mathcal{B}})}\) column vectors of \(\overline{\bm{W}}\), and these column vectors can be manifested as
\begin{equation}\label{Eq_thm_2_3}
\overline{\bm{W}}[:,c_{1}] = \begin{bmatrix} \bm{A}_{1}^{(1)} \otimes \bm{B}_{1}^{(1)} \\ \bm{A}_{1}^{(2)} \otimes \bm{B}_{1}^{(2)} \\ \vdots \\ \bm{A}_{1}^{(d)} \otimes \bm{B}_{1}^{(d)} \end{bmatrix}
\end{equation}
where \(c_{1} \in \{1,2,\cdots,C_{1}^{(\bm{\mathcal{A}})}C_{1}^{(\bm{\mathcal{B}})}\}\). As both Equation (\ref{Eq_thm_2_1}) and (\ref{Eq_general_ktd}) are the format of summation for \(\bm{\mathcal{W}}\), there should be
\begin{equation}\label{Eq_thm_2_4}
\sum_{k=1}^{K}{\bm{\mathcal{A}}_{k} \otimes \bm{\mathcal{B}}_{k}} = \sum_{c=1}^{C^{(\bm{\mathcal{W}}})} {\bm{w}_{c}^{(1)} \circ \bm{w}_{c}^{(2)} \circ \cdots \circ \bm{w}_{c}^{(d)}}.
\end{equation}
Then according to (\ref{Eq_thm_2_3}) and (\ref{Eq_thm_2_4}), since \(\overline{\bm{W}}[:,c_{1}]\) only refers to \(k=1\), there should be
\begin{equation}\label{Eq_thm_2_5}
\bm{\mathcal{A}}_{1} \otimes \bm{\mathcal{B}}_{1} = \sum_{c=1}^{C_{1}^{(\bm{\mathcal{A}})}C_{1}^{(\bm{\mathcal{B}})}} {\bm{w}_{c}^{(1)} \circ \bm{w}_{c}^{(2)} \circ \cdots \circ \bm{w}_{c}^{(d)}}.
\end{equation}
Without loss of generality, \(k=1\) is enough since the similar result to Equation (\ref{Eq_thm_2_5}) could be obtained when \(k \neq 1\) on the basis of linear superposition.

With this simplification of the theorem, we can assume \(K=1\) and just prove that
\begin{equation}\label{Eq_thm_2_6}
\begin{aligned}
&\left(\bm{\mathcal{I}}_{1}^{(\bm{\mathcal{A}})} \bullet \bm{A}_{1}^{(1)} \bullet \cdots \bullet \bm{A}_{1}^{(d)}\right) \otimes \left(\bm{\mathcal{I}}_{1}^{(\bm{\mathcal{B}})} \bullet \bm{B}_{1}^{(1)} \bullet \cdots \bullet \bm{B}_{1}^{(d)}\right) \\
= & \sum_{c=1}^{C_{1}^{(\bm{\mathcal{A}})}C_{1}^{(\bm{\mathcal{B}})}} {\bm{w}_{c}^{(1)} \circ \bm{w}_{c}^{(2)} \circ \cdots \circ \bm{w}_{c}^{(d)}}.
\end{aligned}
\end{equation}

Firstly, consider the left side of Equation (\ref{Eq_thm_2_6}), CP-tensors in contraction format can be rewritten in the sum of rank-1 tensors like
\begin{equation}\label{Eq_thm_2_7}
\left(\sum_{\gamma=1}^{C_{1}^{(\bm{\mathcal{A}})}} {\bm{a}_{\gamma}^{(1)} \circ \bm{a}_{\gamma}^{(2)} \circ \cdots \circ \bm{a}_{\gamma}^{(d)}}\right) \otimes \left(\sum_{\tau=1}^{C_{1}^{(\bm{\mathcal{B}})}} {\bm{b}_{\tau}^{(1)} \circ \bm{b}_{\tau}^{(2)} \circ \cdots \circ \bm{b}_{\tau}^{(d)}}\right)
\end{equation}
where \(\bm{a}_{\gamma}^{(i)} \in \mathbb{R} ^{m_{i}}\) and \(\bm{b}_{\tau}^{(i)} \in \mathbb{R} ^{n_{i}}\). Assume the corresponding indices of \(m_{i}\), \(n_{i}\) and \(m_{i}n_{i}\) are, respectively, \(\alpha_{i}\), \(\beta_{i}\) and \(\omega_{i}\), combining (\ref{Eq_general_ktd}) and (\ref{Eq_thm_2_7}), we have
\begin{equation}\label{Eq_thm_2_8}
\begin{aligned}
&\mathcal{W}(\omega_{1},\omega_{2},\cdots,\omega_{d})=\mathcal{W}(\overline{\alpha_{1}\beta_{1}},\overline{\alpha_{2}\beta_{2}},\cdots,\overline{\alpha_{d}\beta_{d}}) \\
= &\mathcal{A}(\alpha_{1},\alpha_{2},\cdots,\alpha_{d})\mathcal{B}(\beta_{1},\beta_{2},\cdots,\beta_{d}) \\
= &\left(\sum_{\gamma=1}^{C_{1}^{(\bm{\mathcal{A}})}} {a_{\gamma}^{(1)}(\alpha_{1}) \cdots a_{\gamma}^{(d)}(\alpha_{d})}\right) \left(\sum_{\tau=1}^{C_{1}^{(\bm{\mathcal{B}})}} {b_{\tau}^{(1)}(\beta_{1}) \cdots b_{\tau}^{(d)}(\beta_{d})}\right) \\
= &\left(\sum_{\gamma=1}^{C_{1}^{(\bm{\mathcal{A}})}} {\left(\prod_{i=1}^{d} {a_{\gamma}^{(i)}(\alpha_{i})}\right)}\right) \left(\sum_{\tau=1}^{C_{1}^{(\bm{\mathcal{B}})}} {\left(\prod_{i=1}^{d} {b_{\tau}^{(i)}(\beta_{i})}\right)}\right) \\
= & \sum_{\gamma=1}^{C_{1}^{(\bm{\mathcal{A}})}} \sum_{\tau=1}^{C_{1}^{(\bm{\mathcal{B}})}} {\left(\prod_{i=1}^{d} {a_{\gamma}^{(i)}(\alpha_{i})b_{\tau}^{(i)}(\beta_{i})}\right)}.
\end{aligned}
\end{equation}

Secondly, observe the right side of Equation (\ref{Eq_thm_2_6}), from \(\bm{\mathcal{W}}\) to its CP format, we have
\begin{equation}\label{Eq_thm_2_9}
\begin{aligned}
&\mathcal{W}(\omega_{1},\omega_{2},\cdots,\omega_{d}) \\ 
= &\sum_{c=1}^{C_{1}^{(\bm{\mathcal{A}})}C_{1}^{(\bm{\mathcal{B}})}} {w_{c}^{(1)}(\omega_{1}) w_{c}^{(2)}(\omega_{2}) \cdots w_{c}^{(d)}(\omega_{d})} \\ 
= &\sum_{c=1}^{C_{1}^{(\bm{\mathcal{A}})}C_{1}^{(\bm{\mathcal{B}})}} {\left(\prod_{i=i}^{d} {w_{c}^{(i)}(\omega_{i})}\right)}.
\end{aligned}
\end{equation}
Define the \(i\)th sub-matrix in Equation (\ref{Eq_thm_2_3}) as \(\bm{A}_{1}^{(i)} \otimes \bm{B}_{1}^{(i)} = \bm{D}^{(i)} \in \mathbb{R} ^{m_{i}n_{i} \times C_{1}^{(\bm{\mathcal{A}})}C_{1}^{(\bm{\mathcal{B}})}}\), then any single entry \(w_{c}^{(i)}(\omega_{i})\) in Equation (\ref{Eq_thm_2_9}) is actually the entry \(D^{(i)}(\omega_{i},c)\) because sub-matrices in Equation (\ref{Eq_thm_2_3}) are the factor matrices of CP when \(K=1\), i.e., \(\bm{D}^{(i)} = \bm{W}^{(i)}\). So we have
\begin{equation}\label{Eq_thm_2_10}
\begin{aligned}
&w_{c}^{(i)}(\omega_{i}) = D^{(i)}(\omega_{i},c) = D^{(i)}(\overline{\alpha_{i}\beta{i}},\overline{\gamma\tau}) \\
= &A_{1}^{(i)}(\alpha_{i},\gamma)B_{1}^{(i)}(\beta_{i},\tau)
\end{aligned}
\end{equation}
where \(\gamma\) and \(\tau\) are indices of \(C_{1}^{(\bm{\mathcal{A}})}\) and \(C_{1}^{(\bm{\mathcal{B}})}\) respectively.
Similarly, if we rewrite Equation (\ref{Eq_kt_to_kcp}) to the sum of rank-1 tensors, just like \(w_{c}^{(i)}(\omega_{i}) = D^{(i)}(\omega_{i},c)\), we have
\begin{equation}\label{Eq_thm_2_11}
A_{1}^{(i)}(\alpha_{i},\gamma)B_{1}^{(i)}(\beta_{i},\tau) = a_{\gamma}^{(i)}(\alpha_{i})b_{\tau}^{(i)}(\beta_{i}).
\end{equation}
In accordance with simultaneous Equations (\ref{Eq_thm_2_9}), (\ref{Eq_thm_2_10}) and (\ref{Eq_thm_2_11}), we have
\begin{equation}\label{Eq_thm_2_12}
\begin{aligned}
&\mathcal{W}(\omega_{1},\omega_{2},\cdots,\omega_{d}) = \sum_{c=1}^{C_{1}^{(\bm{\mathcal{A}})}C_{1}^{(\bm{\mathcal{B}})}} {\left(\prod_{i=i}^{d} {w_{c}^{(i)}(\omega_{i})}\right)} \\
= &\sum_{\gamma=1}^{C_{1}^{(\bm{\mathcal{A}})}} \sum_{\tau=1}^{C_{1}^{(\bm{\mathcal{B}})}} {\left(\prod_{i=i}^{d} {a_{\gamma}^{(i)}(\alpha_{i})b_{\tau}^{(i)}(\beta_{i})}\right)}.
\end{aligned}
\end{equation}

Finally, according to Equation (\ref{Eq_thm_2_8}) and (\ref{Eq_thm_2_12}), Equation (\ref{Eq_thm_2_6}) is established, then this theorem is proved.
\end{proof}

\subsection{Proof of the Theorem \ref{Thm_Multiply}}\label{appdx:b}
\begin{proof}
In the angle of entries, if we assume the indices of \(m_{i}\), \(n_{i}\), and \(C^{(\bm{\mathcal{W}})}\) are \(\alpha_{i}\), \(\beta_{i}\), and \(c\) respectively, according to Algorithm \ref{Alg_KCP}, we could have
\begin{equation}\label{Eq_thm_3_1}
\begin{aligned}
& \mathcal{Y}(\beta_{1}, \beta_{2}) \\ 
= & \sum_{c=1}^{C^{(\bm{\mathcal{W}})}} \sum_{\alpha_{1}=1}^{m_{1}} \sum_{\alpha_{2}=1}^{m_{2}} \mathcal{X}(\alpha_{1}, \alpha_{2}) W^{(1)}(\overline{\alpha_{1}\beta_{1}},c) W^{(2)}(\overline{\alpha_{2}\beta_{2}},c).
\end{aligned}
\end{equation}
Meanwhile, if we assume the indices of \(C_{k}^{(\bm{\mathcal{A}})}\) and \(C_{k}^{(\bm{\mathcal{B}})}\) are separately \(\gamma_{k}\) and \(\tau_{k}\), Equation (\ref{Eq_RelaxFast}) could be represented as
\begin{equation}\label{Eq_thm_3_2}
\begin{aligned}
& \mathcal{Y}(\beta_{1}, \beta_{2}) \\ 
= & \sum_{k=1}^{K} \sum_{\alpha_{1}=1}^{m_{1}} \sum_{\alpha_{2}=1}^{m_{2}} \sum_{\gamma_{k}=1}^{C_{k}^{(\bm{\mathcal{A}})}} \sum_{\tau_{k}=1}^{C_{k}^{(\bm{\mathcal{B}})}} \mathcal{X}(\alpha_{1}, \alpha_{2}) A_{k}^{(1)}(\alpha_{1}, \gamma_{k}) \\ & B_{k}^{(1)}(\beta_{1}, \tau_{k}) A_{k}^{(2)}(\alpha_{2}, \gamma_{k}) B_{k}^{(2)}(\beta_{2}, \tau_{k}).
\end{aligned}
\end{equation}
Naturally, we can just prove that the right sides of Equation (\ref{Eq_thm_3_1}) and (\ref{Eq_thm_3_2}) are equal.

According to Theorem \ref{Thm_KCP}, the column structure of \(\bm{W}^{(i)}\) defined in Equation (\ref{Eq_W_i}) is
\begin{equation}\label{Eq_thm_3_3}
\underbrace{[ \underbrace{\bm{A}_{1}^{(i)} \otimes \bm{B}_{1}^{(i)}}_{c_1 \in [1, C_{1}^{(\bm{\mathcal{W}})}]} \quad \underbrace{\bm{A}_{2}^{(i)} \otimes \bm{B}_{2}^{(i)}}_{c_2 \in [1, C_{2}^{(\bm{\mathcal{W}})}]} \quad \cdots \quad \underbrace{\bm{A}_{K}^{(i)} \otimes \bm{B}_{K}^{(i)}}_{c_K \in [1, C_{K}^{(\bm{\mathcal{W}})}]} ]}_{c \in [1, C^{(\bm{\mathcal{W}})}]}
\end{equation}
where \(c_{k}\) is the index of the local columns \(C_{k}^{(\bm{\mathcal{A}})}C_{k}^{(\bm{\mathcal{B}})}\) which is represented as \(C_{k}^{(\bm{\mathcal{W}})}\) here, and we have \(c_{k} = \overline{\gamma_{k}\tau_{k}}\).

If we denote \(\bm{D}_{k}^{(i)}=\bm{A}_{k}^{(i)} \otimes \bm{B}_{k}^{(i)}\) as shown in Algorithm \ref{Alg_KCP}, based on (\ref{Eq_thm_3_3}), the contraction of mode \(C^{(\bm{\mathcal{W}})}\) could be regarded as the contractions of modes \(C_{1}^{(\bm{\mathcal{W}})}\), \(C_{2}^{(\bm{\mathcal{W}})}\), \(\cdots\), \(C_{K}^{(\bm{\mathcal{W}})}\). Hence, entries in the right side of Equation (\ref{Eq_thm_3_1}) could be further transformed as
\begin{align}\label{Eq_thm_3_4}
%\begin{aligned}
& \mathcal{Y}(\beta_{1}, \beta_{2}) \notag \\ 
= & \sum_{\alpha_{1}=1}^{m_{1}} \sum_{\alpha_{2}=1}^{m_{2}} \mathcal{X}(\alpha_{1}, \alpha_{2}) \left(\sum_{c=1}^{C^{(\bm{\mathcal{W}})}} W^{(1)}(\overline{\alpha_{1}\beta_{1}},c) W^{(2)}(\overline{\alpha_{2}\beta_{2}},c)\right) \notag \\
= & \sum_{\alpha_{1}=1}^{m_{1}} \sum_{\alpha_{2}=1}^{m_{2}} \mathcal{X}(\alpha_{1}, \alpha_{2}) \left(\sum_{c_{1}=1}^{C_{1}^{(\bm{\mathcal{W}})}} D_{1}^{(1)}(\overline{\alpha_{1}\beta_{1}},c_{1}) D_{1}^{(2)}(\overline{\alpha_{2}\beta_{2}},c_{1}) \right. \notag \\
& \left. + \sum_{c_{2}=1}^{C_{2}^{(\bm{\mathcal{W}})}} D_{2}^{(1)}(\overline{\alpha_{1}\beta_{1}},c_{2}) D_{2}^{(2)}(\overline{\alpha_{2}\beta_{2}},c_{2}) + \cdots + \right. \notag \\
& \left. \sum_{c_{K}=1}^{C_{K}^{(\bm{\mathcal{W}})}} D_{K}^{(1)}(\overline{\alpha_{1}\beta_{1}},c_{K}) D_{K}^{(2)}(\overline{\alpha_{2}\beta_{2}},c_{K})\right) \notag \\
= & \sum_{\alpha_{1}=1}^{m_{1}} \sum_{\alpha_{2}=1}^{m_{2}} \mathcal{X}(\alpha_{1}, \alpha_{2}) \sum_{k=1}^{K} \sum_{c_{k}=1}^{C_{k}^{(\bm{\mathcal{W}})}} D_{k}^{(1)}(\overline{\alpha_{1}\beta_{1}},c_{k}) D_{k}^{(2)}(\overline{\alpha_{2}\beta_{2}},c_{k}) \notag \\
= & \sum_{\alpha_{1}=1}^{m_{1}} \sum_{\alpha_{2}=1}^{m_{2}} \mathcal{X}(\alpha_{1}, \alpha_{2}) \sum_{k=1}^{K} \sum_{\gamma_{k}=1}^{C_{k}^{(\bm{\mathcal{A}})}} \sum_{\tau_{k}=1}^{C_{k}^{(\bm{\mathcal{B}})}} D_{k}^{(1)}(\overline{\alpha_{1}\beta_{1}},\overline{\gamma_{k}\tau_{k}}) \notag \\ 
& D_{k}^{(2)}(\overline{\alpha_{2}\beta_{2}},\overline{\gamma_{k}\tau_{k}}) \notag \\
= & \sum_{\alpha_{1}=1}^{m_{1}} \sum_{\alpha_{2}=1}^{m_{2}} \mathcal{X}(\alpha_{1}, \alpha_{2}) \sum_{k=1}^{K} \sum_{\gamma_{k}=1}^{C_{k}^{(\bm{\mathcal{A}})}} \sum_{\tau_{k}=1}^{C_{k}^{(\bm{\mathcal{B}})}} A_{k}^{(1)}(\alpha_{1}, \gamma_{k}) B_{k}^{(1)}(\beta_{1}, \tau_{k}) \notag \\ 
& A_{k}^{(2)}(\alpha_{2}, \gamma_{k}) B_{k}^{(2)}(\beta_{2}, \tau_{k}) \notag \\
= & \sum_{k=1}^{K} \sum_{\alpha_{1}=1}^{m_{1}} \sum_{\alpha_{2}=1}^{m_{2}} \sum_{\gamma_{k}=1}^{C_{k}^{(\bm{\mathcal{A}})}} \sum_{\tau_{k}=1}^{C_{k}^{(\bm{\mathcal{B}})}} \mathcal{X}(\alpha_{1}, \alpha_{2}) A_{k}^{(1)}(\alpha_{1}, \gamma_{k}) B_{k}^{(1)}(\beta_{1}, \tau_{k}) \notag \\
& A_{k}^{(2)}(\alpha_{2}, \gamma_{k}) B_{k}^{(2)}(\beta_{2}, \tau_{k})
%\end{aligned}
\end{align}
which can conclude this theorem by combining with Equation (\ref{Eq_thm_3_2}).
\end{proof}

\subsection{Proof of the Theorem \ref{Thm_CompComp}}\label{appdx:c}
\begin{proof}
Firstly, according to Algorithm \ref{Alg_KCP_2}, there are totally \(2d\) operations in single branch \(k\), since \(d\) is even. Therefore, just considering \(i=1\) and \(i=2\) is enough to infer the whole computation complexity. As described in Algorithm \ref{Alg_KCP_2} and intuitively shown in Figure \ref{Fig_multiplication}(b), \(C^{(\bm{\mathcal{A}})}\) is introduced when \(\bm{A}_{k}^{(1)}\) is passed, then \(C^{(\bm{\mathcal{B}})}\) is introduced by Kronecker product with \(\bm{B}_{k}^{(1)}\), finally \(C^{(\bm{\mathcal{A}})}\) and \(C^{(\bm{\mathcal{B}})}\) are separately contracted during the following two operations. The separate computation complexity of these four steps are \(\mathcal{O}\left({\rm max}\{m,n\}^{d}C^{(\bm{\mathcal{A}})}\right)\), \(\mathcal{O}\left({\rm max}\{m,n\}^{d}C^{(\bm{\mathcal{A}})}C^{(\bm{\mathcal{B}})}\right)\), \(\mathcal{O}\left({\rm max}\{m,n\}^{d}C^{(\bm{\mathcal{A}})}C^{(\bm{\mathcal{B}})}\right)\), and \(\mathcal{O}\left({\rm max}\{m,n\}^{d}C^{(\bm{\mathcal{B}})}\right)\). Therefore, consider all the \(K\) branches, from \(i=1\) to \(i=2\), the computation complexity is
\begin{equation}\label{Eq_thm_4_1}
\mathcal{O}\left({\rm max}\{m,n\}^{d}\left(C^{(\bm{\mathcal{A}})} + C^{(\bm{\mathcal{B}})} + 2C^{(\bm{\mathcal{A}})}C^{(\bm{\mathcal{B}})}\right)K\right).
\end{equation}
Naturally, the whole corresponding computation complexity of Algorithm \ref{Alg_KCP_2} should be
\begin{equation}\label{Eq_thm_4_2}
\mathcal{O}\left(d{\rm max}\{m,n\}^{d}\left(\frac{1}{2}C^{(\bm{\mathcal{A}})} + \frac{1}{2}C^{(\bm{\mathcal{B}})} + C^{(\bm{\mathcal{A}})}C^{(\bm{\mathcal{B}})}\right)K\right).
\end{equation}

Secondly, for any meaningful tensorizing, there must be \({\rm max}\{m,n\} \geq 2\), which could result
\begin{equation}\label{Eq_thm_4_3}
({\rm max}\{m,n\}-1)C^{(\bm{\mathcal{A}})}C^{(\bm{\mathcal{B}})} \geq C^{(\bm{\mathcal{A}})}C^{(\bm{\mathcal{B}})}.
\end{equation}
In addition, since \(C^{(\bm{\mathcal{A}})} \in \mathbb{N}_{+}\) and \(C^{(\bm{\mathcal{B}})} \in \mathbb{N}_{+}\), there always have
\begin{equation}\label{Eq_thm_4_4}
C^{(\bm{\mathcal{A}})}C^{(\bm{\mathcal{B}})} \geq \frac{1}{2}\left(C^{(\bm{\mathcal{A}})}+C^{(\bm{\mathcal{B}})}\right).
\end{equation}
On the basis of Inequality (\ref{Eq_thm_4_3}) and (\ref{Eq_thm_4_4}), we could have
\begin{equation}\label{Eq_thm_4_5}
{\rm max}\{m,n\}C^{(\bm{\mathcal{A}})}C^{(\bm{\mathcal{B}})} \geq \frac{1}{2}C^{(\bm{\mathcal{A}})} + \frac{1}{2}C^{(\bm{\mathcal{B}})} + C^{(\bm{\mathcal{A}})}C^{(\bm{\mathcal{B}})},
\end{equation}
which obviously implies that
\begin{equation}\label{Eq_thm_4_6}
\begin{aligned}
&\mathcal{O}\left(d{\rm max}\{m,n\}^{d+1}C^{(\bm{\mathcal{A}})}C^{(\bm{\mathcal{B}})}K\right) \\
\geq & \mathcal{O}\left(d{\rm max}\{m,n\}^{d}\left(\frac{1}{2}C^{(\bm{\mathcal{A}})} + \frac{1}{2}C^{(\bm{\mathcal{B}})} + C^{(\bm{\mathcal{A}})}C^{(\bm{\mathcal{B}})}\right)K\right).
\end{aligned}
\end{equation}
\end{proof}

\section*{Acknowledgment}
This work was supported partially by National Key R\&D Program of China (2018YEF0200200), and Beijing Academy of Artificial Intelligence (BAAI), and the Science and Technology Major Project of Guangzhou (202007030006), and the open project of Zhejiang Laboratory.



\begin{thebibliography}{10}
\providecommand{\url}[1]{#1}
\csname url@samestyle\endcsname
\providecommand{\newblock}{\relax}
\providecommand{\bibinfo}[2]{#2}
\providecommand{\BIBentrySTDinterwordspacing}{\spaceskip=0pt\relax}
\providecommand{\BIBentryALTinterwordstretchfactor}{4}
\providecommand{\BIBentryALTinterwordspacing}{\spaceskip=\fontdimen2\font plus
\BIBentryALTinterwordstretchfactor\fontdimen3\font minus
  \fontdimen4\font\relax}
\providecommand{\BIBforeignlanguage}[2]{{%
\expandafter\ifx\csname l@#1\endcsname\relax
\typeout{** WARNING: IEEEtran.bst: No hyphenation pattern has been}%
\typeout{** loaded for the language `#1'. Using the pattern for}%
\typeout{** the default language instead.}%
\else
\language=\csname l@#1\endcsname
\fi
#2}}
\providecommand{\BIBdecl}{\relax}
\BIBdecl

\bibitem{Hochreiter_1997_LSTMInvent}
S.~{Hochreiter} and J.~{Schmidhuber}, ``Long short-term memory,'' \emph{Neural
  Computation}, vol.~9, no.~8, pp. 1735--1780, 1997.

\bibitem{Cho_2014_GRUInvent}
K.~Cho, B.~{van Merrienboer}, C.~Gulcehre, F.~Bougares, H.~Schwenk, and
  Y.~Bengio, ``Learning phrase representations using rnn encoder-decoder for
  statistical machine translation,'' in \emph{Conference on Empirical Methods
  in Natural Language Processing (EMNLP)}, 2014.

\bibitem{Bengio_1994_VanishGradient}
Y.~{Bengio}, P.~{Simard}, and P.~{Frasconi}, ``Learning long-term dependencies
  with gradient descent is difficult,'' \emph{IEEE Transactions on Neural
  Networks}, vol.~5, no.~2, pp. 157--166, 1994.

\bibitem{Greff_2017_LSTM}
K.~{Greff}, R.~K. {Srivastava}, J.~{Koutník}, B.~R. {Steunebrink}, and
  J.~{Schmidhuber}, ``{LSTM}: A search space odyssey,'' \emph{IEEE Transactions
  on Neural Networks and Learning Systems}, vol.~28, no.~10, pp. 2222--2232,
  2017.

\bibitem{Deng_2020_Survey}
L.~Deng, G.~Li, S.~Han, L.~Shi, and Y.~Xie, ``Model compression and hardware
  acceleration for neural networks: A comprehensive survey,'' \emph{Proceedings
  of the IEEE}, vol. 108, no.~4, pp. 485--532, 2020.

\bibitem{Wu_2016_SLSTM}
Z.~{Wu} and S.~{King}, ``Investigating gated recurrent networks for speech
  synthesis,'' in \emph{IEEE International Conference on Acoustics, Speech and
  Signal Processing (ICASSP)}, 2016, pp. 5140--5144.

\bibitem{Zhang_2016_SkipRNN}
S.~Zhang, Y.~Wu, T.~Che, Z.~Lin, R.~Memisevic, R.~R. Salakhutdinov, and
  Y.~Bengio, ``Architectural complexity measures of recurrent neural
  networks,'' in \emph{Proceedings of the 29th International Conference on
  Neural Information Processing Systems (NIPS)}, 2016, pp. 1822--1830.

\bibitem{Hubara_2018_QRNN}
I.~Hubara, M.~Courbariaux, D.~Soudry, R.~El-Yaniv, and Y.~Bengio, ``Quantized
  neural networks: Training neural networks with low precision weights and
  activations,'' \emph{Journal of Machine Learning Research}, vol.~18, no. 187,
  pp. 1--30, 2018.

\bibitem{Dai_2020_PruneRNN}
X.~{Dai}, H.~{Yin}, and N.~K. {Jha}, ``Grow and prune compact, fast, and
  accurate lstms,'' \emph{IEEE Transactions on Computers}, vol.~69, no.~3, pp.
  441--452, 2020.

\bibitem{Novikov_2015_TT}
A.~Novikov, D.~Podoprikhin, A.~Osokin, and D.~P. Vetrov, ``Tensorizing neural
  networks,'' in \emph{Proceedings of the 28th International Conference on
  Neural Information Processing Systems (NIPS)}, 2015, pp. 442--450.

\bibitem{Yang_2017_TTRNN}
Y.~Yang, D.~Krompass, and V.~Tresp, ``Tensor-train recurrent neural networks
  for video classification,'' in \emph{Proceedings of the 34th International
  Conference on Machine Learning (ICML)}, vol.~70, 2017, pp. 3891--3900.

\bibitem{Ye_2018_BTD}
J.~Ye, L.~Wang, G.~Li, D.~Chen, S.~Zhe, X.~Chu, and Z.~Xu, ``Learning compact
  recurrent neural networks with block-term tensor decomposition,'' in
  \emph{IEEE Conference on Computer Vision and Pattern Recognition (CVPR)},
  June 2018, pp. 9378--9387.

\bibitem{Pan_2019_TRRNN}
Y.~Pan, J.~Xu, J.~Ye, M.~Wang, F.~Wang, K.~Bai, and Z.~Xu, ``Compressing
  recurrent neural networks with tensor ring decomposition for action
  recognization,'' in \emph{33rd AAAI Conference on Artificial Intelligence},
  2019.

\bibitem{Yin_2020_HTRNN}
M.~Yin, S.~Liao, X.~Liu, X.~Wang, and B.~Yuan, ``Compressing recurrent neural
  networks using hierarchical tucker tensor decomposition,'' \emph{arXiv
  preprint arXiv:2005.04366}, 2020.

\bibitem{Cichocki_2018_TensorNetworks}
A.~Cichocki, ``Tensor networks for dimensionality reduction, big data and deep
  learning,'' in \emph{Advances in Data Analysis with Computational
  Intelligence Methods}, ser. Studies in Computational Intelligence.\hskip 1em
  plus 0.5em minus 0.4em\relax Springer International Publishing AG, 2018, vol.
  738, pp. 3--49.

\bibitem{Tjandra_2017_TTRNN1}
A.~Tjandra, S.~Sakti, and S.~Nakamura, ``Compressing recurrent neural network
  with tensor train,'' in \emph{International Joint Conference on Neural
  Networks (IJCNN)}, 2017, pp. 4451--4458.

\bibitem{Oseledets_2011_InventTT}
I.~Oseledets, ``Tensor-train decomposition,'' \emph{SIAM Journal on Scientific
  Computing}, vol.~33, no.~5, pp. 2295--2317, 2011.

\bibitem{Liu_2011_UCF11}
J.~Liu, J.~Luo, and M.~Shah, ``Recognizing realistic actions from videos “in
  the wild”,'' in \emph{IEEE Conference on Computer Vision and Pattern
  Recognition (CVPR)}, 2009, pp. 1996--2003.

\bibitem{Kim_2008_YCF}
M.~{Kim}, S.~{Kumar}, V.~{Pavlovic}, and H.~{Rowley}, ``Face tracking and
  recognition with visual constraints in real-world videos,'' in \emph{IEEE
  Conference on Computer Vision and Pattern Recognition (CVPR)}, 2008, pp.
  1--8.

\bibitem{DeLathauwer_2008_InventBTD}
L.~De~Lathauwer, ``Decompositions of a higher-order tensor in block terms ---
  {Part II}: Definitions and uniqueness,'' \emph{SIAM Journal on Matrix
  Analysis and Applications}, vol.~30, no.~3, pp. 1033--1066, 2008.

\bibitem{Zhao_2018_TR}
Q.~Zhao, M.~Sugiyama, L.~Yuan, and A.~Cichocki, ``Learning efficient tensor
  representations with ring structure networks,'' in \emph{6th International
  Conference on Learning Representations (ICLR)}, 2018.

\bibitem{Wu_2020_Hybrid}
B.~Wu, D.~Wang, G.~Zhao, L.~Deng, and G.~Li, ``Hybrid tensor decomposition in
  neural network compression,'' \emph{arXiv preprint arXiv:2006.15938}, 2020.

\bibitem{Grasedyck_2010_InventHT}
L.~Grasedyck, ``Hierarchical singular value decomposition of tensors,''
  \emph{SIAM Journal on Matrix Analysis and Applications}, vol.~31, no.~4, p.
  2029–2054, 2010.

\bibitem{Phan_2013_KTD2}
A.~H. {Phan}, A.~{Cichocki}, P.~{Tichavský}, R.~{Zdunek}, and S.~{Lehky},
  ``From basis components to complex structural patterns,'' in \emph{IEEE
  International Conference on Acoustics, Speech and Signal Processing
  (ICASSP)}, 2013, pp. 3228--3232.

\bibitem{Phan_2012_KTD1}
A.~H. Phan, A.~Cichocki, P.~Tichavský, D.~P. Mandic, and K.~Matsuoka, ``On
  revealing replicating structures in multiway data: A novel tensor
  decomposition approach,'' in \emph{International Conference on Latent
  Variable Analysis and Signal Separation (LVA/ICA)}, 2012, pp. 297--305.

\bibitem{Espig_2011_TensorGraph}
M.~Espig, W.~Hackbusch, S.~Handschuh, and R.~Schneider, ``Optimization problems
  in contracted tensor networks,'' \emph{Computing and Visualization in
  Science}, vol.~14, no.~6, p. 271–285, 2011.

\bibitem{Cichocki_2016_TensorBook}
A.~Cichocki, N.~Lee, I.~Oseledets, A.-H. Phan, Q.~Zhao, and D.~P. Mandic,
  ``Tensor networks for dimensionality reduction and large-scale optimization:
  Part 1 low-rank tensor decompositions,'' \emph{Foundations and
  Trends{\textregistered} in Machine Learning}, vol.~9, no. 4-5, pp. 249--429,
  2016.

\bibitem{Dolgov_2014_MultiIndex}
S.~V. Dolgov and D.~V. Savostyanov, ``Alternating minimal energy methods for
  linear systems in higher dimensions,'' \emph{SIAM Journal on Scientific
  Computing}, vol.~36, no.~5, pp. A2248--A2271, 2014.

\bibitem{Hillar_2013_NPHard}
C.~J. Hillar and L.-H. Lim, ``Most tensor problems are {NP-Hard},''
  \emph{Journal of the ACM}, vol.~60, no.~6, pp. 45.1--45.39, 2013.

\bibitem{Khrulkov_2018_ExpPowerRNN}
V.~Khrulkov, A.~Novikov, and I.~Oseledets, ``Expressive power of recurrent
  neural networks,'' in \emph{6th International Conference on Learning
  Representations (ICLR)}, 2018.

\bibitem{Phan_2020_CP1Vector}
A.~{Phan}, A.~{Cichocki}, I.~{Oseledets}, G.~G. {Calvi}, S.~{Ahmadi-Asl}, and
  D.~P. {Mandic}, ``Tensor networks for latent variable analysis: Higher order
  canonical polyadic decomposition,'' \emph{IEEE Transactions on Neural
  Networks and Learning Systems}, vol.~31, no.~6, pp. 2174--2188, 2020.

\bibitem{Varol_2018_LongTerm3DCNN}
G.~Varol, I.~Laptev, and C.~Schmid, ``Long-term temporal convolutions for
  action recognition,'' \emph{IEEE Transactions on Pattern Analysis and Machine
  Intelligence}, vol.~40, no.~6, pp. 1510--1517, 2018.

\bibitem{Sharma_2016_UCF11}
S.~Sharma, R.~Kiros, and R.~Salakhutdinov, ``Action recognition using visual
  attention,'' in \emph{4th International Conference on Learning
  Representations (ICLR)}, 2016.

\bibitem{Lu_2015_MultiManifold}
J.~{Lu}, G.~{Wang}, W.~{Deng}, P.~{Moulin}, and J.~{Zhou}, ``Multi-manifold
  deep metric learning for image set classification,'' in \emph{IEEE Conference
  on Computer Vision and Pattern Recognition (CVPR)}, 2015, pp. 1137--1145.

\bibitem{Li_2016_RRNN}
Y.~Li, W.~Zheng, and Z.~Cui, ``Recurrent regression for face recognition,''
  \emph{arXiv preprint arXiv:1607.06999}, 2016.

\bibitem{Banerjee_2017_VGGFPool}
B.~{Banerjee} and V.~{Murino}, ``Efficient pooling of image based cnn features
  for action recognition in videos,'' in \emph{IEEE International Conference on
  Acoustics, Speech and Signal Processing (ICASSP)}, 2017, pp. 2637--2641.

\bibitem{Xu_2019_SDMA_TDD}
Z.~{Xu}, R.~{Hu}, J.~{Chen}, C.~{Chen}, J.~{Jiang}, J.~{Li}, and H.~{Li},
  ``Semisupervised discriminant multimanifold analysis for action
  recognition,'' \emph{IEEE Transactions on Neural Networks and Learning
  Systems}, vol.~30, no.~10, pp. 2951--2962, 2019.

\bibitem{Wang_2019_Fusion}
D.~{Wang}, J.~{Yang}, and Y.~{Zhou}, ``Human action recognition based on
  multi-mode spatial-temporal feature fusion,'' in \emph{22th International
  Conference on Information Fusion (FUSION)}, 2019, pp. 1--7.

\end{thebibliography}
\end{document}